\documentclass{article}

\def\arxivver{1}

\usepackage{fullpage}
\usepackage{amsmath}
\usepackage{amsthm}
\usepackage{mathtools}
\usepackage{thmtools,thm-restate}
\usepackage[round]{natbib}
\usepackage[allcolors=blue,colorlinks]{hyperref}

\usepackage[extdef=true]{delimset}
\usepackage{algorithm}
\usepackage{algorithmic}
\usepackage{xspace}

\usepackage{enumitem}
\setlist[itemize]{itemsep=0ex}
\setlist[enumerate]{itemsep=0ex,label=(\roman*)}

\usepackage{mytimes}

\declaretheoremstyle[
	    spaceabove=\topsep, 
	    spacebelow=\topsep, 
	    headfont=\normalfont\bfseries,
	    bodyfont=\normalfont\itshape,
	    notefont=\normalfont\bfseries,
	    notebraces={(}{)},
	    postheadspace=0.33em, 
	    headpunct={.},
    ]{theorem}
\declaretheorem[style=theorem]{theorem}

\declaretheoremstyle[
	    spaceabove=\topsep, 
	    spacebelow=\topsep, 
	    headfont=\normalfont\bfseries,
	    bodyfont=\normalfont,
	    notefont=\normalfont\bfseries,
	    notebraces={(}{)},
	    postheadspace=0.33em, 
	    headpunct={.},
    ]{definition}

\declaretheoremstyle[
        spaceabove=\topsep, 
        spacebelow=\topsep, 
        headfont=\normalfont\bfseries,
        bodyfont=\normalfont,
        notefont=\normalfont\bfseries,
        notebraces={}{},
        postheadspace=0.33em, 
        qed=$\blacksquare$, 
        headpunct={.},
    ]{proofstyle}
\declaretheorem[style=proofstyle,numbered=no,name=Proof]{proof}

\declaretheoremstyle[
	    spaceabove=\topsep, 
	    spacebelow=\topsep, 
	    headfont=\normalfont\bfseries,
	    bodyfont=\normalfont\itshape,
	    notefont=\normalfont\bfseries,
	    notebraces={}{},
	    postheadspace=0.33em, 
	    headpunct={.},
    ]{restatement-style}
\declaretheorem[unnumbered,style=restatement-style,name=\ignorespaces]{restatement}

\declaretheorem[style=theorem,sibling=theorem,name=Lemma]{lemma}

\declaretheorem[style=theorem,numbered=no,name=Theorem]{theorem*}
\declaretheorem[style=theorem,numbered=no,name=Lemma]{lemma*}
\declaretheorem[style=theorem,numbered=no,name=Corollary]{corollary*}
\declaretheorem[style=theorem,numbered=no,name=Proposition]{proposition*}
\declaretheorem[style=theorem,numbered=no,name=Claim]{claim*}
\declaretheorem[style=theorem,numbered=no,name=Fact]{fact*}
\declaretheorem[style=theorem,numbered=no,name=Observation]{observation*}
\declaretheorem[style=theorem,numbered=no,name=Conjecture]{conjecture*}

\declaretheorem[style=definition,sibling=theorem,name=Definition]{definition}

\declaretheorem[style=definition,numbered=no,name=Definition]{definition*}
\declaretheorem[style=definition,numbered=no,name=Remark]{remark*}
\declaretheorem[style=definition,numbered=no,name=Example]{example*}
\declaretheorem[style=definition,numbered=no,name=Question]{question*}

\usepackage[capitalise]{cleveref}
\usepackage{crossreftools}

\crefname{ALC@unique}{line}{lines}

\newcommand{\regret}{\text{Reg}_K}
\newcommand{\regT}{\text{Reg}_T}
\newcommand{\wh}{\widehat}
\newcommand{\wt}{\widetilde}
\newcommand{\ind}[1]{\mathbb{I}\{#1\}}

\newcommand{\bbR}{\mathbb{R}}
\newcommand{\bbE}{\mathbb{E}}
\newcommand{\cK}{\mathcal{K}}
\newcommand{\cS}{\mathcal{S}}

\newcommand{\breg}{B_R\delimpair{(}{[.]\,\|\,}{)}}
\DeclareMathOperator*{\poly}{poly}
\DeclareMathOperator*{\argmin}{arg\,min}

\newcommand{\dotp}{\boldsymbol{\cdot}}
\newcommand{\Tr}{\operatorname{Tr}}
\newcommand{\tr}{^\mathsf{T}}
\newcommand{\half}{\frac{1}{2}}
\newcommand{\ifrac}[2]{#1/#2}
\newcommand{\ipfrac}[2]{(#1)/(#2)}

\newcommand{\blbname}{distorted linear bandits\xspace}
\newcommand{\blbnamecaps}{Distorted Linear Bandits\xspace}
\newcommand{\blbacronym}{DLB\xspace}

\title{Online Markov Decision Processes with Aggregate Bandit Feedback}
\author{
    Alon Cohen
    \thanks{Google Research, Tel Aviv; \href{mailto:aloncohen@google.com}{\texttt{aloncohen@google.com}}.}
    \and
    Haim Kaplan
    \thanks{Tel-Aviv University and Google Research, Tel Aviv; \href{mailto:haimk@post.tau.ac.il}{\texttt{haimk@post.tau.ac.il}}.}
    \and
    Tomer Koren
    \thanks{Tel-Aviv University and Google Research, Tel Aviv; \href{mailto:tkoren@tauex.tau.ac.il}{\texttt{tkoren@tauex.tau.ac.il}}.}
    \and
    Yishay Mansour
    \thanks{Tel-Aviv University and Google Research, Tel Aviv; Supported in part by a grant from the ISF. \href{mailto:mansour@tau.ac.il}{\texttt{mansour@tau.ac.il}}.}
}
\date{\today}

\begin{document}

\maketitle

\begin{abstract}
    We study a novel variant of online finite-horizon Markov Decision Processes with adversarially changing loss functions and initially unknown dynamics. In each episode, the learner suffers the loss accumulated along the trajectory realized by the policy chosen for the episode, and observes \emph{aggregate bandit feedback}: the trajectory is revealed along with the cumulative loss suffered, rather than the individual losses encountered along the trajectory.
    %
    %
    Our main result is a computationally efficient algorithm with $O(\sqrt{K})$ regret for this setting, where $K$ is the number of episodes.
    
    We establish this result via an efficient reduction to a novel bandit learning setting we call \blbnamecaps (\blbacronym), which is a variant of bandit linear optimization where actions chosen by the learner are adversarially distorted before they are committed.
    We then develop a computationally-efficient online algorithm for \blbacronym for which we prove an $O(\sqrt{T})$ regret bound, where $T$ is the number of time steps.
    Our algorithm is based on online mirror descent with a self-concordant barrier regularization that employs a novel increasing learning rate schedule.
\end{abstract}

\section{Introduction}
    
    Markov Decision Processes are a ubiquitous model for decision making that captures a wide array of applications including autonomous road navigation, robotics, gaming and many more. In the finite-horizon version of the model, the goal of the agent is to minimize her expected total loss over a fixed number of time steps. Classic results in finite-horizon MDPs state that the optimal policy of the agent is deterministic; namely, a mapping between each state and time step to an action for the agent to play.

    In this paper, we study the problem of \emph{Online MDPs with Aggregate Feedback} which is played for $K$ episodes. The dynamics of the MDP are fixed but unknown to the learner. After each episode, in addition to observing her trajectory within the MDP, the agent also gets to view her total loss along this trajectory. The agent, however, does not get to observe the individual losses of specific states and actions that comprise the trajectory.
    This setting was recently considered in \cite{efroni2020reinforcement} where the authors derived computationally-efficient learning algorithms for the case where the losses are sampled i.i.d.\ from some unknown distribution.
    In this work, we assume that the losses are non-stochastic and may be chosen by an adversary---a significantly more challenging task.
    
    The adversarial setting is a variant of online MDPs~\citep{even2009online} with initially unknown model dynamics, previously considered either when full information about the losses is received~\citep{neu2010online}, or with traditional bandit feedback where the agent sees the individual losses of all states and actions that were visited along each of her generated trajectories~\citep{rosenberg2019online}. 
    Commonly, the main solution technique is to separate the $K$ episodes into $O(\log K)$ epochs; in each epoch, the agent runs a no-regret algorithm using an estimate of the dynamics obtained from observations accrued up to the beginning of the epoch. 
    To tackle bandit feedback in general, it is common practice to employ a full-information learning algorithm which is fed with an unbiased estimate of the losses in each episode.
    Nevertheless, in our setup we do not know the MDP dynamics, so it is hopeless for the learner to generate such an unbiased estimate since it is impossible to calculate the probability of visiting each state and action without exact knowledge of the transition distributions.
    This impediment was overcome in \cite{jin2020learning} that followed the ``optimism in the face of uncertainty'' principle: they fed the learning algorithm with a certain underestimate of the loss. 
    This drives the agent to explore under-sampled state-action pairs, helps to obtain better estimates of the dynamics, and reduces the overall bias of the loss estimators over time.
    
    We utilize a similar approach to tackle the aggregate feedback by reducing the problem to $O(\log K)$ epochs in each of which we solve a variant of linear bandits over our current estimate of the model dynamics. We name the learning problem in each epoch  \emph{\blbnamecaps} (\blbacronym). This is a variant of the linear bandits problem in which, after choosing an action, it can be distorted (i.e., perturbed) in an adversarial manner before it is played. This distortion unavoidably introduces a non-negligible bias when trying to generate an estimate of the loss vector.
    The \blbacronym problem is also interesting in its own right, capturing scenarios where there is uncertainty regarding the action that is actually taken, which might deviate significantly from the action intended to be taken---a phenomenon that occurs in applications in robotics and control, where the actions are continuous in~nature.
    
    We derive two learning algorithms for the \blbacronym setting that yield a $O(\sqrt{K})$ regret bound, yet mitigate the estimation bias in different ways.
    Our first algorithm, based on EXP2 \citep{awerbuch2004adaptive, mcmahan2004online}, utilizes an optimistic approach by feeding the algorithm with underestimates of the loss.
    This technique, it turns out, is not computationally efficient due to the non-convex nature of these underestimates.
    Our second algorithm, however, runs in polynomial-time per episode.
    It is a variant of Online Mirror Descent that uses a self-concordant barrier function as a regularizer \citep{abernethy2009competing} with a series of increasing learning rates.
    The idea of using increasing learning rates to alleviate estimation bias is used in various recent works \citep{bubeck2017kernel, agarwal2017corralling, lee2020bias}.
    Intuitively, it gives the learner a ``boost'' towards playing better actions whenever the estimation bias is large.
    
    \subsection{Summary of Contributions}
    
        The main contributions of the paper are as follows:
        \begin{itemize}
            \item We introduce the setting of Online MDPs with Aggregate Bandit Feedback, where the dynamics are initially unknown and costs may be chosen by an adversary; to the best of our knowledge, such a problem has not been studied before.
            \item We establish an efficient reduction from Online MDPs with Aggregate Bandit Feedback to the novel {\blbnamecaps} (\blbacronym) problem.
            \item We give a computationally-efficient online learning algorithm for the \blbacronym problem with $O(\sqrt{T})$ regret over $T$ rounds.
            \item Combining the two techniques, we obtain a computationally-efficient online learning method for Online MDPs with Aggregate Bandit Feedback with $O(\sqrt{K})$ regret over $K$ episodes.
        \end{itemize}
    
        In \cref{sec:reduction-main} we present the Online MDP with Aggregate Feedback model explicitly and give our main result. We also present the \blbnamecaps (\blbacronym) setting, the reduction between the two models, and prove our regret bound for online MDPs. In \cref{sec:blb-algorithms} we give our two algorithms for the \blbacronym setting and analyze their regrets.

    \subsection{Additional Related Work}
    
        The study of regret minimization in reinforcement learning dates back to \cite{jaksch2010near}
        who considered an MDP with unknown dynamics and losses, but where the losses are sampled i.i.d. This model was further studied in \cite{azar2017minimax,zanette2019tighter} that provided improved~bounds.
        
        Online MDPs were introduced in
        \cite{even2009online} who studied MDPs with known dynamics and adversarially changing losses. Later \cite{neu2013online} extended the online MDP to handle bandit feedback.
        \cite{abbasi2013online} considered MDPs where both the dynamics and the losses change adversarially. Their algorithm, however, is not computationally-efficient as they show in a hardness result. All the above results assume access to individual losses while in this work we assume the learner observes only the aggregate loss of an episode.

        Bandit linear optimization has been extensively studied under both semi-bandit and bandit feedback; for an extensive survey of this literature, see \citet{SlivkinsBook,CsabaTorBook}. 
        %
        %
        Misspecified linear bandits were introduced in \cite{GhoshCG2017} where the loss of each action can be perturbed arbitrarily. They give an impossibility result for large sparse deviations and a regret bound for small deviations (see also \citealp{LattimoreSW2020}).
        Our model differs from misspecified linear bandit, most importantly, in that we allow for adversarial losses. In addition, we also differ both in the fact that the deviations might be large (and we can only globally bound them) and the fact that the loss is linear but with respect to a distorted action.

\section{Preliminaries} \label{sec:prelims}

    \paragraph{Finite-Horizon MDPs.}
    
    A finite-horizon Markov Decision Process is a tuple $(S,A,s_1,P,\ell,H)$ defined as follows. $S$ is a finite set of states; $A$ is a finite set of actions; $s_1\in S$ is the start state; the integer $H$ defines the horizon.
    The transition function $P$ defines a probability distribution $P(s' \mid s,a,h)$ of the next state $s'$ given the current state $s$, action $a$, and time $h\in[H]$. The loss function is $\ell$ defines a loss $\ell(h, s,a,s') \in [0,1]$ for every time $h\in[H]$ state $s$, action $a$, and next state $s'$.
    
    A (randomized) policy $\pi : S \times [H] \mapsto \Delta(A)$ maps each state and time to a probability distribution over the actions.
    A trajectory is a sequence  $(s_1, a_1, \ldots, s_H, a_{H}, s_{H+1})$. 
    The probability of such trajectory with respect to a policy $\pi$ and a transition function $P$ is $\prod_{h=1}^H \pi(a_h \mid s_h,h) P(s_{h+1} \mid s_h,a_h,h)$.
    The accumulated loss of such a trajectory using a loss function $\ell$ is $\sum_{h=1}^{H} \ell(h, s_h,a_h,s_{h+1})$. The expected loss of a policy $\pi$ with respect to a transition function $P$ and loss function $\ell$ is 
    \begin{align*}
        L^{\pi,P,\ell}
        = 
        \bbE \brk[s]4{
        \sum_{h=1}^{H}\ell(h, s_h,a_h,s_{h+1})}
        =
        \sum_{\substack{(h,s,a,s') \in \\ [H] \times S \times A \times S}} \ell(h, s,a,s') ~ \Pr[s_h=s,a_h=a,s_{h+1}=s']
        .
    \end{align*}
    
    \paragraph{Occupancy Measures.}
        
        A combination of a policy $\pi$ and a transition function $P$ provide an occupancy measure $x^{\pi, P}$ such that $x^{\pi, P}(h,s,a,s')$ is the probability, according to $P$ and $\pi$, of being at state $s$ at time $h$, playing action $a$, and transitioning to state $s'$. 
        Formally,
        \[
            x^{\pi,P}(h, s,a,s')
            =
            \Pr_{\pi,P}[s_h=s,a_h=a,s_{h+1}=s'].
        \]
        Any $x : [H] \times S \times A \times S \mapsto \bbR$ is an occupancy measure, if and only if
        \begin{alignat}{3}
            &x(h,s,a,s') \ge 0,&
            \qquad 
            &\forall (h,s,a,s') \in [H] \times S \times A \times S,& \nonumber \\
            &\sum_{\mathclap{(s,a,s') \in S \times A \times S}} x(h,s,a,s') = 1,& 
            \qquad 
            &\forall h \in [H],& \nonumber\\
            &\sum_{\mathclap{(a,s') \in A \times S}} x(h+1,s,a,s') = \sum_{\mathclap{(s',a) \in S \times A}} x(h,s',a,s),& 
            \qquad 
            &\forall (s,h) \in S \times [H-1].& \label{eq:omset}
        \end{alignat}
        Indeed, any $x$ that satisfies the conditions above corresponds to an occupancy measure for some policy $\pi$ and transition function $P$, both can easily be extracted from $x$---this correspondence is therefore one-to-one. 
        That is, given an occupancy measure $x$ we can define the corresponding policy and dynamics as follows:
        \begin{equation} \label{eq:oc-pnd}
            \pi^{(x)}(a \mid s,h) = \frac{\sum_{s' \in S} x(h,s,a,s')}{\sum_{(a,s') \in A \times S} x(h,s,a,s')},
            \quad \text{and} \quad
            \wt P^{(x)}(s' \mid s,a,h) = \frac{x(h,s,a,s')}{\sum_{s' \in S} x(s,a,h,s')}.
        \end{equation}
        For more on occupancy measures, see \cite{rosenberg2019online}. 

    \paragraph{Self-concordant Barriers and Bregman Divergence.} \label{sec:scb}
    
    We next briefly review self-concordant barrier functions---a fundamental tool in interior-point methods that was also shown to be highly-useful in linear bandit optimization \citep{abernethy2009competing}.
    Self-concordant barriers are discussed in-depth in \cite{nemirovski2004interior}; we give the technical definitions in \cref{app:self} and here focus on some useful properties of such functions that we use.

    
    
    We consider a $\vartheta$-self-concordant barrier function $R$ over 
    a convex set $\cS$. 
    In particular, for a self-concordant barrier $R$,
    the function 
    $\|h\|_x = \sqrt{h\tr\nabla^2 R(x)h}$ is a norm, and also $\nabla R : \text{int}(\cS) \mapsto \bbR^d$ is invertible. 
    In addition, an important property of the norm $\|\cdot\|_x$ 
    is that for any point $y \in \bbR^d$ and $x \in \text{int}\brk{\cS}$,
    \begin{equation} \label{eq:dikininbody}
        \|y-x\|_x < 1 
        ~\implies~ 
        y \in \text{int}\brk{\cS}.
    \end{equation}
    
    We define the Bregman divergence with respect to a $\vartheta$-self-concordant barrier $R$ as follows:
    \[ 
        \breg{y}{x} = R(y) - R(x) - \nabla R(x) \dotp (y-x).
    \]
    The Bregman divergence is always nonnegative: $\breg{y}{x} \geq 0$ for any $x,y \in \text{int} \brk{\cS}$. 
    Moreover, we shall need the following lower bound on the Bregman divergence (see \citealp{nemirovski2004interior}):
    \begin{equation} 
        \breg{y}{x} 
        \geq 
        \rho \brk{\norm{y-x}_{x}} \quad \text{for} \quad \rho(z) = z - \log(1+z)
        . 
        \label{eq:scbergmanlb}
    \end{equation}
    %
    %
    We also require the following lemma whose proof is found in \cref{app:proofs}. 
    \begin{lemma} \label{lem:boundbregfromx1}
        Define $\cS_\gamma = \{(1-\gamma)x + \gamma x_1 \mid x \in \cS\}$ for $x_1 = \argmin_{x \in \cS} R(x)$ and some $\gamma \in [0,1]$. 
        Then $\breg{y}{x_1} \le \vartheta \log(1/\gamma)$ for any $y \in \cS_\gamma$.
    \end{lemma}
    
    \paragraph{Online Mirror Descent with Barriers.}
    
    We rely on standard properties of the Online Mirror Descent (OMD) algorithm with a self-concordant barrier function $R$ for a domain $\cS$ as regularization, applied to an arbitrary sequence of loss vectors $\ell_1,\ldots,\ell_T \in \bbR^d$ \citep{abernethy2009competing}.
    Starting from an initial $x_1 \in \cS$, OMD makes the following updates for $t=1,\ldots,T$:\footnote{Typically, OMD has an additional projection step when employed on a bounded domain. However, when $R$ is a barrier, such a projection is redundant as the OMD update never steps out of the domain (this is a consequence of \cref{eq:dikininbody}).}
    \begin{align} \label{ALG:OMD}
        x_{t+1} &= \nabla R^{-1} \brk!{ \nabla R(x_t) - \eta_t \ell_t }.
    \end{align}
    %
    %
    This version of OMD has the following guarantee (we include a proof in \cref{app:proofs} for completeness); here we use the notation $\norm{\ell}_x^\star = \sqrt{\ell\tr \nabla^2 R(x)^{-1} \ell}$ for $x \in \cS$ and $\ell \in \bbR^d$.
    \begin{lemma} \label{thm:mdbarrier}
        Let $R : \text{int}(\cS) \mapsto \bbR$ be self-concordant and assume that $\eta_t \norm{\ell_t}_{x_t}^\star \leq \tfrac12$ for all $t$.
        Then, for any $u \in \cS$,
        \[
            \sum_{t=1}^T \ell_t \dotp (x_t - u)
            \le 
            \frac{1}{\eta_1} \breg{u}{x_1}
            -
            \sum_{t=2}^T \brk3{\frac{1}{\eta_{t-1}} - \frac{1}{\eta_t}} \breg{u}{x_t}
            + 
            \sum_{t=1}^T \eta_t \brk{\norm{\ell_t}_{x_t}^\star}^2
            .
        \]
    \end{lemma}

    Observe that when the learning rate sequence is strictly increasing, the middle term in the above bound becomes negative and can potentially serve to decrease the regret of OMD, particularly when the divergence $\breg{u}{x_t}$ is large. This observation will be key to our algorithmic development in \cref{sec:solution-efficient}.
 
\section{Setup and Overview of Results} \label{sec:reduction-main}

    \subsection{Online MDPs with Aggregate Bandit Feedback}
        
        We consider an online version of finite-horizon MDPs in which the interaction between learner and the MDP proceeds for $K$ episodes. Before the interaction begins, the environment assigns a sequence of loss functions $\ell_1,\ldots,\ell_k : [H] \times S \times A \times S \mapsto [0,1]$ one for each episode $k\in [K]$. The choice of loss functions is done in an arbitrary, possibly adversarial, manner.
        
        At the start of each episode $k$ the online algorithm defines a policy $\pi_k$.
        At the end of the episode the online algorithm receives the trajectory realized by $\pi_k$, 
        i.e.,  $(s^k_1, a^k_1, \ldots, s^k_H, a^k_H, s^k_{H+1})$,
        and the aggregate loss incurred during this trajectory with respect to $\ell_k$, i.e., $\sum_{h=1}^{H} \ell_k(h, s^k_h,a^k_h,s^k_{h+1})$.

        We define the regret of the learner over the $K$ episodes as
        \[
            \regret = \sum_{k=1}^K L^{\pi_k, P,\ell_k}  - \min_{\pi} \sum_{k=1}^K L^{\pi, P,\ell_k} ,
        \]
        where the minimum is taken over all policies $\pi$, and we let $\pi^\star$ denote a minimizer.
        The regret can also be written in terms of occupancy measures, by noticing that the expected loss of a policy $\pi$ and transition function $P$ with respect to a loss function $\ell$ is $L^{\pi,P,\ell}=x^{\pi,P}\dotp \ell$.
        Thus, the regret of the learner over the $K$ episodes can be written as:
        \[
            \regret 
            = 
            \sum_{k=1}^K x^{\pi_k, P} \dotp \ell_k - \min_{\pi} \sum_{k=1}^K x^{\pi, P} \dotp \ell_k
            .
        \]
        
        The main result of this paper is a computationally-efficient learning algorithm for the setting described above.
        \begin{theorem} \label{thm:main}
            There exists an online learning algorithm for \emph{finite-horizon MDPs with aggregated bandit feedback} that guarantees
            \[
                \bbE \brk[s]{\regret} = \poly\brk{H,|S|,|A|} \, O(\sqrt{K})
                .
            \]
            Moreover, the per-episode runtime complexity of the algorithm is polynomial in $H, |S|, |A|$, and $K$.
        \end{theorem}
        
        We prove the theorem by efficiently reducing the online MDPs setting to a sequence of instances of a novel setting we term \emph{\blbnamecaps} (\blbacronym).
        In what follows, we describe the \blbacronym setting, the reduction, and  prove the correctness of the reduction. 

\subsection{\blbnamecaps (\blbacronym)} \label{sec:blb}

    %
    In this game, the learner plays by picking vectors from a compact and convex body $\cS \subseteq \bbR^d$. We assume that $\norm{y}_1 \leq H$ for all $y \in \cS$ for some $H>0$. Further, let $\beta>0$ be a bias parameter.
    Learning in the \blbacronym setting proceed according the following protocol: 
    Initially, the adversary privately chooses a sequence of loss vectors $\ell_1,\ldots,\ell_T$ and a sequence of perturbation vectors $\epsilon_1,\ldots,\epsilon_T \in [0,\beta]^d$.
    Then, at rounds $t = 1,\ldots,T$,
    \begin{enumerate}
        \item Learner selects $y_t \in \cS$.
        \item Adversary picks $z_t \in \bbR^d$, where $\norm{z_t}_1 \le H$ such that $\|z_t - y_t\|_1 \leq \min\{\abs{z_t \dotp \epsilon_t}, \abs{y_t \dotp \epsilon_t} \}$.
        \item A random $\hat z_t$ is sampled such that $\bbE_t \brk[s]{\hat z_t \mid z_t} = z_t$ and $\norm{\hat z_t}_1 \le H$, where $\bbE_t$ denotes expectation conditioned on all randomness prior to round $t$.
        \item The action $\hat z_t$ is played; the learner suffers and observes the loss $\ell_t \dotp \hat z_t$; the learner additionally observes $\hat z_t$ and $\epsilon_t$.
    \end{enumerate}
    We emphasize that the $z_t$ are arbitrary and can be chosen in an adaptive manner after the learner chooses $y_t$. Note, however, that we assume that the loss vectors (as well as the perturbation vectors) are chosen before the game starts; namely, the adversary is oblivious.
    
    We define the regret in the \blbacronym setting as follows: 
    \[
        \regT = \sum_{t=1}^T \hat z_t \dotp \ell_t - \min_{z \in \cS} \sum_{t=1}^T z \dotp \ell_t.  
    \]
    The learner's goal is therefore to minimize the losses attained by the perturbations $\hat z_1, \ldots, \hat z_T$ of the  actions   $y_1, \ldots, y_T$ chosen by the learner. Clearly, the regret necessarily scales with the magnitudes of $\epsilon_1, \ldots, \epsilon_T$, and our regret bounds will ultimately depend on a parameter $B$ that upper bounds the magnitude of the perturbations via the quantity $\sum_{t=1}^T (\hat z_t \dotp \epsilon_t)^2$.
    The following theorem is the main technical result of our work.
    
    \begin{theorem} \label{thm:blbregret}
        There exists an efficient (poly-time) online learning algorithm for the \blbacronym setting whose regret is at most
        $
            \poly(H, d, \beta, B) \, O(\sqrt{T}).
        $
    \end{theorem}
    
    We prove this theorem by showing two online learning algorithms (one is computationally-efficient; the other is not) in \cref{sec:blb-algorithms}. We conclude this current section by describing the reduction from online MDPs with aggregate bandit feedback to \blbacronym.

\subsection{The Reduction} \label{sec:reduction}

    We now show how to reduce the  MDP with aggregate feedback problem to instances of \blbacronym described above (proofs of results of this section appear in \cref{app:proofs}.)
    Our algorithm for learning  MDPs with aggregate feedback is depicted in detail in \cref{app:reduction-algo}, and here we give a verbal description of the algorithm. The algorithm assumes the existence of a computationally-efficient online learning algorithm for \blbacronym with $O(\sqrt{T})$ regret which exists due to \cref{thm:blbregret}.
    
    The algorithm partitions the $K$ episodes into epochs, where  epoch $i$ contains episodes $k_i$ through $k_{i+1}-1$  ($k_1 = 1$). Each epoch ends whenever the number of visits to some state-action pair $s,a$ at some time step $h$ is doubled. Thus, the total number of epochs is at most $2 H |S| |A| \log K$. 
    
    In epoch $i$, we produce an empirical estimate of the transition probabilities $\wh P$ based on all observations prior to epoch $i$. We apply a high probability argument to bound the estimation error $P$ of the dynamics as: $\norm{\wh P(\cdot \mid s,a,h) - P(\cdot \mid s,a,h)}_1 \le \epsilon_i(s,a,h)/H$ for a confidence parameter $\epsilon_i : S \times A \times [H] \mapsto [0, \beta]$ associated with epoch $i$. ($\epsilon_i(s,a,h)$ decreases as a function of the number of times each $(s,a,h)$ has been visited up to epoch $i$.)
    
    We fix a convex and compact $\cS_i$ to be the set of all feasible occupancy measures based on our current estimate of the dynamics of the MDP. We claim that in each epoch, the setting admits to the \blbname problem.
    Indeed, we show that with high probability, $\cS_i$ contains $x^{\pi^\star, P}$---the occupancy measure associated with the optimal policy and the true dynamics.
    Now, throughout epoch $i$, for $k = k_i,\ldots,k_{i+1}-1$:
    \begin{enumerate}
      \item Learner picks a policy $\pi_k$ associated with some occupancy measure $y_k \in \cS_i$.
      \item $\pi_k$ is played on the true MDP and the learner observes a trajectory $\hat z_k$, such that $\hat z_k(h,s,a,s') = 1$ iff the trajectory passed through state $s$ at time $h$, played action $a$ and transitioned to state $s'$. Otherwise $\hat z_k(h,s,a,s') = 0$. The learner suffers and observes the loss of $\ell_k \dotp \hat z_k$.
      \item Let $z_k$ be the occupancy measure of $\pi_k$ and the true dynamics $P$; then $\bbE[\hat z_k \mid z_k] = z_k$. We prove that $\norm{z_k-y_k}_1 \le \min\{\epsilon_i \dotp z_t, \epsilon_i \dotp y_t\}$.
    \end{enumerate}
    Moreover, we give a bound of $\sum_{k=k_i}^{k_{i+1}-1} (\epsilon_i \dotp \hat z_k)^2 = \wt O\brk{H^4 |S|^2 |A|}$ as required by the \blbacronym setting.
    
    We consequently apply the \blbacronym algorithm to obtain a regret bound of $O(\sqrt{k_{i+1}-k_i}) = O(\sqrt{K})$ in each epoch, and as the number of epochs is only at most $O(\log K)$ this gives an overall regret bound of $\wt O(\sqrt{K})$ as required.
    The complete proof of this claim appears in \cref{app:proofs}. The
    analysis of the running time of the algorithm is found in \cref{app:efficientreduction}. 

\section{Algorithms for \blbnamecaps}
\label{sec:blb-algorithms}

In this section we prove \cref{thm:blbregret} by presenting our online algorithms for the \blbacronym problem.
The difficulty of this setting lies in the fact that the main mechanism to cope with lack of information in bandit optimization is to construct unbiased estimates of the loss vectors. In the \blbacronym setting this is impossible to do since the actions chosen by the learner are shifted by the adversary. 
Nevertheless, having $\sum_{t=1}^T (\hat z_t \dotp \epsilon_t)^2$ bounded, intuitively means that the estimation bias at the actions played by the learner is bounded in an amortized sense---a useful property that we utilize in our algorithms.

\subsection{Simple Approach via Optimism}
\label{sec:solution-optimism}

Our first algorithm is based on what is arguably the most straightforward approach to the problem: construct an ``optimistic'' estimator to the player's loss---one whose expectation underestimates the loss of all actions at a given round, yet is sufficiently accurate in estimating the player's loss at the same round---and feed it to a standard bandit linear optimization algorithm. 
However, as we show in this section, such a loss estimator becomes a non-convex (in fact, concave) function of the played action, thus overall this approach leads to a computationally inefficient algorithm.


Throughout this section, we assume that the decision set $\cS$ is finite of size $O((H T)^d)$; since for now we are not bound by computational complexity considerations, if $\cS$ is a larger (or infinite) set we may replace $\cS$ with a $1/(HT)$-net of $\cS$, which has the required size.
The algorithm we describe below (\cref{alg:inefficient}) assumes as input an exploration distribution $\mu$ over the set $\cS$, such that for $y \sim \mu$ it holds that $\bbE[yy\tr] \succeq \lambda I$ for a constant $\lambda > 0$. 
Standard techniques in linear bandit optimization (e.g., \citealp{bubeck2012towards,hazan2016volumetric}) show that under fairly general conditions on $\cS$, one can pick an exploration distribution $\mu$ so as $\lambda = \Omega(1/\sqrt{d})$.%
\footnote{Some of these techniques rely on solving intractable optimization problems, but recall that in the context of this section we are not concerned by the computational complexity of the resulting algorithm.}

\begin{algorithm}[ht]
    \caption{\blbnamecaps via Optimistic Biases}
    \begin{algorithmic}[1]
        \STATE {\bf input:} $\eta > 0$, $\gamma > 0$, exploration distribution $\mu$.
        \STATE {\bf initialize:} $w_1(y) = 1$ for all $y \in \cS$.
        \FOR{$t = 1,\ldots,T$}
            \STATE {\bf define} probability density $p_t \propto w_t$, and let $q_t = (1-\gamma)p_t + \gamma \mu$.
            \STATE {\bf sample} point $y_t \sim q_t$ in the domain $\cS$. 
            \STATE {\bf adversary} chooses $z_t$ such that $\norm{z_t-y_t}_1 \le \min\brk[c]{\abs{y_t \dotp \epsilon_t}, \abs{z_t \dotp \epsilon_t}}$, and plays $\hat z_t$ where $\bbE_t \brk[s]{\hat z_t \mid z_t} = z_t$.
            \STATE {\bf observe} $\epsilon_t$ and loss $\ell_t \dotp \hat z_t \in [0,H]$.
            \STATE {\bf compute} the second moment of $y_t$:
            $
                M_t = \bbE_t[y_t y_t\tr ].
            $
            \STATE {\bf compute} $\hat\ell_t = (\ell_t \dotp \hat z_t) M_t^{-1} y_t$ 
            and $\tilde\ell_t(y) = \hat\ell_t \dotp y - \sqrt{d} \, \norm{y}_{M_t^{-1}} \norm{\epsilon_t}_{M_t}$.
            
            \STATE {\bf update} $w_{t+1}(y) = w_t(y) \cdot
            \exp\brk{ -\eta \tilde\ell_t(y)}, \qquad \forall ~ y \in \cS$. 
        \ENDFOR
    \end{algorithmic}
    \label{alg:inefficient}
\end{algorithm}

The algorithm relies on a standard estimator $\hat\ell_t = (\ell_t \dotp \hat z_t) M_t^{-1} y_t$ to estimate the loss vector $\ell_t$.
Note that if it were that $z_t=y_t$ then this would have been an unbiased estimator for the loss, i.e., $\bbE_t [\hat\ell_t]=\ell_t$. 
However, due to the adversarial perturbations $z_t$ might be shifted away from the intended $y_t$. We thus modify the estimator to account for this shift and make it ``optimistic,'' in the sense that its expectation is a lower bound on the real loss function.
Given these corrected estimates, the rest of the algorithmic development follows standard lines in the linear bandit optimization literature~\citep{dani2008price,bubeck2012towards}.

Concretely, we define the following \emph{bias-corrected} loss functions:
\begin{align*}
    \tilde\ell_t(y)
    =
    \hat\ell_t \dotp y - \sqrt{d} \, \norm{y}_{M_t^{-1}} \norm{\epsilon_t}_{M_t}
    ,
    \qquad
    \forall ~ t \in [T], y \in \cS
    ~.
\end{align*}
Then, the algorithm essentially performs multiplicative-weights updates on the modified loss functions $\tilde\ell_t(y)$, which can be seen to be a \emph{concave} function of $y$.
In general, it is a hard problem to sample from the resulting distributions $q_t$ given that these losses are concave. (If, on the other hand, they were convex, then the resulting distributions would have been log-concave for which efficient sampling algorithms are well-known.)
Therefore, the algorithm is computationally inefficient.

We prove that \cref{alg:inefficient} provides the following regret guarantee.

\begin{theorem} \label{thm:optimismmain}
    Set $\eta = (2H\beta d)^{-1} \sqrt{\log |\cS| /T}$, $\gamma = 2H^2 (H + \beta\sqrt{d}) \eta/\lambda$. 
    Then, given that $B \geq \sum_{t=1}^T (\hat z_t \dotp \epsilon_t)^2$ (almost surely),  for any $y^\star \in \cS$, \cref{alg:inefficient} satisfies that
    \begin{align*}
        \bbE\brk[s]4{ \sum_{t=1}^T \ell_t \dotp (\hat z_t - y^\star) }
        = 
        \widetilde O\brk*{ H\beta d + \beta d \sqrt{B} + \frac{H^3}{\beta\lambda d} + \frac{H^2}{\lambda \sqrt{d}} } \sqrt{T}
        ,
    \end{align*}
    provided that $\beta \ge 1$ and $T \ge \ipfrac{4 H^2 (H + \beta \sqrt{d})^2 \log |\cS|}{\lambda^2 \beta^2 d^2}$.
\end{theorem}

\noindent
We only sketch the proof here, deferring details and precise bounds to \cref{app:proofs}.

\begin{proof}[(sketch)]
    We begin by showing that $\tilde \ell_t$ is indeed an underestimate of the true loss (see \cref{lem:loss-underestimate} below):
    \[
        \bbE_t[\tilde\ell_t(y)] \leq \ell_t \dotp y \quad \text{for any} \quad y \in \cS
        .
    \]
    For the converse direction, we show that in expectation over the learner's decision, $\tilde \ell_t$ is close to $\ell_t$ in the following sense:
    \[
        \bbE_t\brk[s]*{ \sum_{y \in \cS} p_t(y) \tilde\ell_t(y) } 
        \ge 
        \bbE_t \brk[s]{\ell_t \dotp \hat z_t} - \gamma H - 5d \norm{\epsilon_t}_{M_t}
        .
    \]
    With these two results at hand, we argue that the regret of \cref{alg:inefficient} is bounded by the regret of the \textsc{Multiplicative Weights} updates, plus an additive error term that scales with the perturbations~$\epsilon_t$:
    \begin{align}
        \bbE\brk[s]4{ \sum_{t=1}^T \ell_t \dotp (\hat z_t - y^\star) }
        \leq
        \bbE\brk[s]4{ \sum_{t=1}^T \sum_{y \in \cS} p_t(y) \brk!{ \tilde\ell_t(y) - \tilde\ell_t(y^\star) } } 
        +
        \gamma H T
        +
        5d\,\bbE\brk[s]4{ \sum_{t=1}^T \norm{\epsilon_t}_{M_t} }
        .
        \label{eq:optimistic-regret-relation}
    \end{align}
    
    Next, we apply a standard second-order regret bound of \textsc{Multiplicative Weights}
    to obtain the following:
    \begin{align*}
        \sum_{t=1}^T \sum_{y \in \cS} p_t(y) \brk!{ \tilde\ell_t(y) - \tilde\ell_t(y^\star) }
        \leq
        \frac{\log\abs{\cS}}{\eta} + \eta \sum_{y \in \cS} p_t(y) \brk!{\tilde\ell_t(y)}^2
        ,
    \end{align*}
    and we bound the term 
    $
        \bbE_t \brk[s]{\sum_{y \in \cS} p_t(y) \brk{\tilde\ell_t(y)}^2}
        \le
        8(H\beta d)^2
    $ 
    using simple algebra.
    
    The theorem is now given by combining the second-order regret bound above together with \cref{eq:optimistic-regret-relation}, and by bounding the bias terms using the \blbacronym setting assumptions, as
    \begin{align*}
        \bbE\brk[s]4{ \sum_{t=1}^T \norm{\epsilon_t}_{M_t} }
        \le
        2\beta \sqrt{B T}
        .
        &\qedhere
    \end{align*}
\end{proof}

We now prove that as mentioned, the expectation of $\tilde \ell_t$ is an underestimate of the true loss.

\begin{lemma} \label{lem:loss-underestimate}
    $\bbE_t[\tilde\ell_t(y)] \leq \ell_t \dotp y$ for any $y \in \cS$.
\end{lemma}

\begin{proof}
    Observe that
    \begin{align*}
        \bbE_t[ \hat\ell_t ]
        &=
        \bbE_t[ (\ell_t \dotp z_t) M_t^{-1} y_t]
        \\
        &=
        \bbE_t[M_t^{-1} y_t y_t\tr \ell_t] + \bbE_t[M_t^{-1} y_t (z_t-y_t) \dotp \ell_t]
        \\
        &=
        \ell_t + M_t^{-1} \bbE_t[y_t (z_t-y_t) \dotp \ell_t]
        .
    \end{align*}
    Our assumptions imply that $\abs{(z_t-y_t) \dotp \ell_t} \leq
    \norm{z_t-y_t}_1 \norm{\ell_t}_\infty \leq \abs{y_t \dotp \epsilon_t}$.
    Thus, by two applications of Cauchy-Schwartz, for any $y \in \cS$ we
    obtain 

    \begin{align} \label{eq:tildell}
        \begin{aligned}
        \bbE_t [ \abs{(\hat\ell_t - \ell_t) \dotp y} ]
           &=
        \bbE_t\brk[s]!{ \abs{y\tr M_t^{-1} y_t} \cdot \abs{(z_t-y_t) \dotp \ell_t} }
        \\
        &\leq
        \norm{y}_{M_t^{-1}} \bbE_t\brk[s]!{ \norm{y_t}_{M_t^{-1}} \abs{y_t \dotp \epsilon_t} }
        \\
        &\leq
        \norm{y}_{M_t^{-1}} \sqrt{ \bbE_t\brk[s]{ \norm{y_t}_{M_t^{-1}}^2 } \, \bbE_t\brk[s]{ (y_t \dotp \epsilon_t)^2 } }
        \\
        &=
        \norm{y}_{M_t^{-1}} \sqrt{ \bbE_t\brk[s]{ y_t\tr M_t^{-1} y_t } } \sqrt{ \epsilon_t\tr \bbE_t[y_t y_t\tr]  \epsilon_t  }
        \\
        &=
        \sqrt{d} \norm{y}_{M_t^{-1}} \norm{\epsilon_t}_{M_t}
        .
        \end{aligned}
    \end{align}
    This means that
    \begin{align*}
        \bbE_t[\tilde\ell_t(y) - \ell_t \dotp y]
        =
        \bbE_t[(\hat\ell_t-\ell_t) \dotp y] - \sqrt{d} \norm{y}_{M_t^{-1}} \norm{\epsilon_t}_{M_t}
        \leq 
        0
    \end{align*}
    which proves that $\bbE[\tilde\ell_t(y)] \leq \ell_t \dotp y$ for any $y \in
    \cS$.
\end{proof}

\subsection{Efficient Approach via OMD with Increasing Learning Rates}
\label{sec:solution-efficient}

Our previous algorithm enjoys an $\wt O(\sqrt{T})$ regret bound, but it is inherently computationally inefficient. 
In this section we take a different approach that leads to an algorithm with $\wt O(\sqrt{T})$ regret, but one that can also be implemented efficiently. 

\begin{algorithm}[ht]
    \caption{\blbnamecaps via Increasing Learning Rates}
    \label{alg:increasinglearningrates}
    \begin{algorithmic}[1]
        \setcounter{ALC@unique}{0} 
        \STATE {\bf input}: $\eta_0 > 0$, $\vartheta$-self-concordant barrier $R : \text{int}\brk{\cS} \mapsto \bbR$.
        \STATE {\bf init}: $x_1 = \argmin_{x \in \cS} R(x)$.
        \FOR{$t = 1,\ldots,T$}
            \STATE {\bf sample} $u_t$ uniformly at random from the unit sphere of $\bbR^d$.
            \STATE {\bf predict} $y_t = x_t + \nabla^2 R(x_t)^{-1/2} u_t$. \label{ln:yt}
            \STATE {\bf adversary} chooses $z_t$ such that $\norm{z_t-y_t}_1 \le \min\brk[c]{\abs{y_t \dotp \epsilon_t}, \abs{z_t \dotp \epsilon_t}}$, and plays $\hat z_t$ where $\bbE_t \brk[s]{\hat z_t \mid z_t} = z_t$.
            \STATE {\bf observe} $\hat z_t$, $\epsilon_t$, and loss $\ell_t \dotp \hat z_t \in [0,H]$.
            \STATE {\bf construct} $\tilde \ell_t = d (\ell_t \dotp \hat z_t) \nabla^2 R(x_t)^{1/2} u_t$. \label{ln:estimator}
            \STATE {\bf update} $\eta_t^{-1} = \eta_{t-1}^{-1} - 2 d \, |\hat z_t \dotp \epsilon_t|$. \label{ln:increase-lr}
            \STATE {\bf set} $x_{t+1} = \nabla R^{-1} \brk{\nabla R(x_t) - \eta_t \tilde \ell_t}$. \label{ln:update}
        \ENDFOR
    \end{algorithmic}
\end{algorithm}

\cref{alg:increasinglearningrates} is based on Online Mirror Descent with a self-concordant barrier $R$ as a regularizer \citep{abernethy2009competing}.
The algorithm maintains a sequence of points $x_1,\ldots,x_T \in \cS$. 
In \cref{ln:yt}, the algorithm makes a prediction $y_t$ by sampling uniformly at random from the ellipsoid $\{y : \norm{y - x_t}_{x_t} \le 1\}$, known as the \emph{Dikin Ellipsoid} associated with $R$ at $x_t$, that is always contained in $\cS$ (this follows from \cref{eq:dikininbody}).
Then, according to the \blbacronym protocol, the algorithm receives $\epsilon_t$, $\hat z_t$ and loss $\ell_t \dotp \hat z_t$ such that $z_t = \bbE_t[\hat z_t \mid z_t]$ where $z_t$ is a perturbation of $y_t$. 

The algorithm proceeds to construct an estimator $\tilde \ell_t$ of the loss vector $\ell_t$ in \cref{ln:estimator}. Note that if we replace $\hat z_t$ with $y_t$ in \cref{ln:estimator}, then $\tilde \ell_t$ would be an unbiased estimator. 
However, as this is not the case, the algorithm must mitigate the bias in the $\tilde \ell_t$, and does that by increasing its learning rate according to the perturbation magnitude $|\hat z_t \dotp \epsilon_t|$ (\cref{ln:increase-lr}).
Finally, in \cref{ln:update}, the algorithm performs the mirror descent update.%

\cref{alg:increasinglearningrates} can be implemented efficiently as long as $\cS$ is not degenerate (namely, $\cS$ is compact and has volume in $\bbR^d$, and thus admits a proper self-concordant barrier $R$) and as long as gradients and Hessians of $R$ can be computed efficiently. 
We defer a more detailed discussion of implementation issues to \cref{app:efficientreduction}.

Our main result regarding the algorithm is as follows.

\begin{theorem} \label{thm:inclrregret}
    \cref{alg:increasinglearningrates} with 
    $
        \eta_0 
        = 
        \widetilde\Theta\brk{ 
        \vartheta / \brk{d \vartheta \sqrt{B T} + d H \sqrt{\vartheta T}}}
    $ 
    provides the following regret guarantee, for any $y^\star \in \cS$:
    \[
        \bbE \brk[s]4{\sum_{t=1}^T (\hat z_t - y^\star) \dotp \ell_t}
        =
        \widetilde O\brk!{
        d \vartheta \sqrt{B T}
        + 
        d H \sqrt{\vartheta T}}
        ,
    \]
    provided that $B \geq \max\{\sum_{t=1}^T (\hat z_t \dotp \epsilon_t)^2, H\}$ (almost surely).
\end{theorem}

Here we sketch the proof of \cref{thm:inclrregret} highlighting the key ideas; the complete proof and precise bounds can be found in \cref{app:proofs}.

\begin{proof}[(sketch)]
    The first part of the proof is straightforward.
    We split the regret into three terms:
    \begin{align}
        \bbE \brk[s]4{\sum_{t=1}^T (\hat z_t-y^\star)\dotp \ell_t}
        =
        \bbE \brk[s]4{\sum_{t=1}^T (z_t-x_t)\dotp \ell_t}
        +
        \bbE \brk[s]4{\sum_{t=1}^T (x_t-y^\star_\gamma)\dotp \ell_t}
        +
        \bbE \brk[s]4{\sum_{t=1}^T (y^\star_\gamma-y^\star)\dotp \ell_t}
        ,
        \label{eq:regretdecomp}
    \end{align}
    where $y^\star_\gamma = (1-\gamma) y^\star + \gamma x_1 \in \cS_\gamma$ for sufficiently small $\gamma$.
    (Following a standard technique, we introduce $y^\star_\gamma$ as otherwise we would eventually have to bound $\breg{y^\star}{x_1}$ which might be arbitrarily large; by introducing $y^\star_\gamma$, we instead would have to bound $\breg{y^\star_\gamma}{x_t}$, which is bounded by \cref{lem:boundbregfromx1}.)
    The first summand in \cref{eq:regretdecomp} pertains to the bias generated by the perturbation of $y_t$ to $z_t$, and is bounded by $\sqrt{B T}$;
    the third summand bounds the loss difference between that of $y^\star$ and of $y^\star_\gamma$, and is bounded by $2 \gamma H T$. All of these quantities are  $\widetilde O\brk!{ d \vartheta \sqrt{B T} + d H \sqrt{\vartheta T}}$.
    
    The heart of the proof focuses on bounding the second summand.
    To this end, we apply \cref{lem:boundlossbias} (see below) to bound the instantaneous regret of the algorithm at each time step $t$, by the instantaneous regret using the loss estimator $\tilde \ell_t$ plus an additional bias term that scales with $\norm{x_t - y^\star_\gamma}_{x_t}$. This results with
    \begin{equation} \label{eq:specificlossbias}
        (x_t - y^\star_\gamma) \dotp \ell_t 
        \le 
        \bbE_t \brk[s]{(x_t - y^\star_\gamma) \dotp \tilde \ell_t} 
        + 
        d \norm{x_t - y^\star_\gamma}_{x_t} \bbE_t \brk[s]{\abs{\hat z_t \dotp \epsilon_t}},
    \end{equation}
    and we proceed in bounding $\bbE [\sum_{t=1}^T \bbE_t \brk[s]{(x_t - y^\star_\gamma) \dotp \tilde \ell_t} ] = \bbE [\sum_{t=1}^T (x_t - y^\star_\gamma) \dotp \tilde \ell_t ]$. 
    Since the algorithm is taking \textsc{OMD} steps with loss vectors $\tilde\ell_t$, we can apply \cref{thm:mdbarrier} to get
    \[
        \bbE \brk[s]4{\sum_{t=1}^T (x_t - y^\star_\gamma) \dotp \tilde \ell_t} 
        \le
        \bbE \brk[s]4{
            \frac{1}{\eta_1} \breg{y^\star_\gamma}{x_1}
            -
            \sum_{t=2}^T \brk3{\frac{1}{\eta_{t-1}} - \frac{1}{\eta_t}} \breg{y^\star_\gamma}{x_t}
            + 
            \sum_{t=1}^T \eta_t (\|\tilde \ell_t\|_{x_t}^\star)^2
        }.
    \]
    Handling the first and third terms is standard (following \citealp{abernethy2009competing}), and they are shown to be bounded by $O((\vartheta/\eta_0) \log(1/\gamma))$ and $O(\eta_0 d^2 H^2 T)$ respectively, both are $\wt O\brk{ d \vartheta \sqrt{B T} + d H \sqrt{\vartheta T}}$ for our choice of parameters.
    The middle term in the bound above is what enables the algorithm to compensate for the bias in the loss estimation by employing an increasing learning rate schedule.
    Indeed, together with \cref{eq:specificlossbias} we obtain
    \begin{align}
    \begin{aligned}
        \bbE \brk[s]4{\sum_{t=1}^T (x_t - y^\star_\gamma) \dotp \ell_t}
        \le\;
        &d \, \bbE\brk[s]*{ \sum_{t=1}^T \norm{x_t - y^\star_\gamma}_{x_t} \abs{\hat z_t \dotp \epsilon_t} }
        -
        \bbE\brk[s]*{ \sum_{t=2}^T \brk3{\frac{1}{\eta_{t-1}} - \frac{1}{\eta_t}} \breg{y^\star_\gamma}{x_t} }
        \\
        &+ 
        \widetilde O\brk!{ d \vartheta \sqrt{B T} + d H \sqrt{\vartheta T}}
        . 
    \end{aligned}
    \label{eq:xxxxx}
    \end{align}
    The key observation is that the divergence $\breg{y^\star_\gamma}{x_t}$ here is directly related to the bias term $\norm{y^\star_\gamma - x_t}_{x_t}$ via \cref{lem:boundbregfromxt} (found below), as
    $$
        \breg{y^\star_\gamma}{x_t}
        \geq 
        \tfrac12 \norm{y^\star_\gamma - x_t}_{x_t} - 1
        .
    $$
    Now, with our particular setting of learning rates (\cref{ln:increase-lr}) the second term in \cref{eq:xxxxx} is upper bounded by 
    $
        -d \sum_{t=1}^T \norm{y^\star_\gamma - x_t}_{x_t} \bbE_t[|\hat z_t \dotp \epsilon_t|] 
        + O((\vartheta/\eta_0) \log(1/\gamma))
    $
    (in expectation), which precisely cancels out the first summation over the bias terms and gives the $\wt O\brk!{ d \vartheta \sqrt{B T} + d H \sqrt{\vartheta T}}$ regret bound.
\end{proof}

The following lemma bounds the instantaneous regrets suffered by the algorithm, by the algorithm's estimates of the instantaneous regret plus an additive bias term that scales as $\norm{x_t-x}_{x_t}$.

\begin{lemma} \label{lem:boundlossbias}
    Let $x \in \cS$. Then,
    $
        (x_t - x) \dotp \ell_t 
        \le 
        \bbE_t \brk[s]{(x_t - x) \dotp \tilde \ell_t} 
        + 
        d \norm{x_t - x}_{x_t} \bbE_t \brk[s]{\abs{\hat z_t \dotp \epsilon_t}}.
    $
\end{lemma}

\begin{proof}
Recall that $x_t$ is determined given the randomness up to time $t$.
We have that 
    \begin{align*}
        \bbE_t \brk[s]1{(x_t - x) \dotp \tilde \ell_t}
        &=
        \bbE_t \brk[s]1{(x_t - x) \dotp d \brk{\ell_t \dotp \hat z_t} \nabla^2 R(x_t)^{1/2} u_t} \\
        &=
        \bbE_t \brk[s]1{(x_t - x) \dotp d \brk1{\ell_t \dotp \bbE_t \brk[s]{\hat z_t \mid z_t}} \nabla^2 R(x_t)^{1/2} u_t} \\
        &=
        \bbE_t \brk[s]1{(x_t - x) \dotp d \brk{\ell_t \dotp z_t} \nabla^2 R(x_t)^{1/2} u_t} \\
        &=
        \underbrace{\bbE_t \brk[s]1{(x_t - x) \dotp d \brk{\ell_t \dotp x_t} \nabla^2 R(x_t)^{1/2} u_t}}_{(1)}
        +
        \underbrace{\bbE_t \brk[s]1{(x_t - x) \dotp d \brk1{\ell_t \dotp \brk{y_t - x_t}} \nabla^2 R(x_t)^{1/2} u_t}}_{(2)} \\
        &\qquad -
        \underbrace{\bbE_t \brk[s]1{(x_t - x) \dotp d \brk1{\ell_t \dotp \brk{y_t - z_t}} \nabla^2 R(x_t)^{1/2} u_t}}_{(3)}.
    \end{align*}
    Next, we analyze each of the three summands above.
    As the only randomness given the history up to time $t$ is in $u_t$, we have
    $
        (1) = 0
    $,
    and as $y_t-x_t = \nabla^2 R(x_t)^{-1/2} u_t$, we have
    \begin{align*}
        (2)
        &=
        \bbE_t \brk[s]1{(x_t - x) \dotp d \nabla^2 R(x_t)^{1/2} u_t \, \brk{y_t - x_t} \dotp \ell_t} \\
        &=
        \bbE_t \brk[s]1{(x_t - x) \dotp d \, \nabla^2 R(x_t)^{1/2} u_t u_t\tr \nabla^2 R(x_t)^{-1/2} \ell_t} \\
        &=
        (x_t - x) \dotp d \, \nabla^2 R(x_t)^{1/2} \bbE_t \brk[s]1{u_t u_t\tr} \nabla^2 R(x_t)^{-1/2} \ell_t \\
        &=
        (x_t - x) \dotp d \, \nabla^2 R(x_t)^{1/2} \dotp \tfrac{1}{d} I \dotp \nabla^2 R(x_t)^{-1/2} \ell_t \\
        &=
        (x_t - x) \dotp \ell_t.
    \end{align*}
    For term (3), two applications of H\"{o}lder's inequality yield
    \begin{align*}
        (3)
        \le
        d \, \bbE_t \brk[s]1{\norm{x_t - x}_{x_t} \norm{\nabla^2 R(x_t)^{1/2} u_t}_{x_t}^\star  
        \norm{\ell_t}_\infty \norm{y_t - z_t}_1}. 
    \end{align*}
    
    Now, to obtain the lemma, we use our assumption that $\norm{\ell_t}_\infty \le 1$, that
    \[
        \bbE_t \brk[s]{\norm{y_t - z_t}_1} 
        \le 
        \bbE_t \brk[s]{\abs{z_t \dotp \epsilon_t}}
        =
        \bbE_t \brk[s]1{\abs{\bbE_t \brk[s]{\hat z_t \mid z_t} \dotp \epsilon_t}}
        \le
        \bbE_t \brk[s]{\abs{\hat z_t \dotp \epsilon_t}},
    \] 
    by Jensen's inequality, and finally
    $
        \norm{R(x_t)^{1/2} u_t}_{x_t}^\star 
        =
        1
    $
    due to \cref{lem:technicalstuff} (see \cref{app:proofs}).
\end{proof}


The next lemma lower bounds the Bregman divergence of any point $x \in \cS$ from $x_t$ by an order of their distance in local norm; i.e., $\|x-x_t\|_{x_t}$.

\begin{lemma} \label{lem:boundbregfromxt}
    Let $x \in \cS$. Then, 
    $
        \breg{x}{x_t} 
        \geq 
        \half \|x-x_t\|_{x_t} - 1.
    $
\end{lemma}

\begin{proof}
    Recall that
    $
        \breg{x}{x_t} 
        \geq 
        \rho \brk{\|x - x_t\|_{x_t}}
    $ 
    by \cref{eq:scbergmanlb} where $\rho(z)=z-\log (1+z)$.
    Since $\rho \brk{z}$ is convex, we can lower bound
    \[
        \rho \brk{z} 
        \ge 
        \rho \brk{1} + \rho' \brk{1} \dotp \brk{z - 1} 
        =
        \half - \log(2) + \half z
        \ge 
        \half z - 1,
    \]
    which yields the lemma's statement for $z=\|x - x_t\|_{x_t}$. 
\end{proof}

\ifx\arxivver\undefined
\else
    \if\arxivver1
        \bibliographystyle{plainnat}
    \fi
\fi
\bibliography{fbmdps.bib}

\appendix

\section{Reduction Algorithm} \label{app:reduction-algo}

    \begin{algorithm}[H]
        \caption{Reduction from online MDPs with aggregate feedback to \blbacronym}
        \label{alg:reduction}
        \begin{algorithmic}[1]
            \STATE {\bf Init}: $N_1(s,a,h) = 0, \, N(s,a,h,s') = 0, \quad \forall (s,a,h,s') \in S \times A \times [H] \times S$, $k = 1$.
            \FOR{epoch $i=1,2,\ldots$}
                \STATE {\bf construct} empirical transition function: 
                \begin{equation} \label{eq:empirical-dynamics}
                    \wh P_i(s' \mid s,a,h) = \frac{N_i(s,a,h,s')}{\max\{N_i(s,a,h),1\}},
                    \quad
                    \forall (s,a,h,s') \in S \times A \times [H] \times S.
                \end{equation}
                \STATE {\bf set} confidence bounds:
                \begin{equation} \label{eq:epsilon}
                    \epsilon_i(s,a,h) = 5H \sqrt{\frac{|S| + \log(H |S| |A| K / \delta)}{\max\{N_i(s,a,h),1\}}},
                    \quad
                    \forall (s,a,h) \in S \times A \times [H].
                \end{equation}
                \STATE {\bf construct} polytope of feasible occupancy measure $\cS_i$ (\cref{eq:omconcentration,eq:ominitial,eq:omflow,eq:omdistribution}).
                \STATE {\bf init}: $n_i(s,a,h) = 0, \, n_i(s,a,h,s') = 0$ ~\textbf{for all}~ $(s,a,h,s') \in S \times A \times [H] \times S$.
                \WHILE{$n_i(s,a,h) < \max\{N_i(s,a,h),1\}$ ~\textbf{for all}~ $(s,a,h) \in S \times A \times [H]$}
                    \STATE {\bf predict} occupancy measure $y_k \in \cS_i$ using algorithm from \cref{thm:blbregret}.
                    \STATE {\bf play} $\pi_k$ such that $\pi_k = \pi^{(y_k)}$ (recall \cref{eq:oc-pnd}).
                    \STATE {\bf observe} trajectory $\hat z_k$, and aggregate loss $\ell_k \dotp \hat z_k$.
                    \STATE {\bf feed} \blbacronym algorithm with $\hat z_k$, $\ell_k \dotp \hat z_k$, and $\epsilon_i$.
                    \STATE {\bf increment}: $k = k+1$, $n_i(s,a,h,s') = n_i(s,a,h,s') + \hat z_k(s,a,h,s')$, and $n_i(s,a,h) = \sum_{s' \in S} n_i(s,a,h,s')$.
                \ENDWHILE
                \STATE {\bf update}: $N_{i+1}(s,a,h) = N_i(s,a,h) + n_i(s,a,h)$, $N_{i+1}(s,a,h,s') = N_i(s,a,h,s') + n_i(s,a,h,s')$.
            \ENDFOR
        \end{algorithmic}
    \end{algorithm}

\section{Deferred Proofs} \label{app:proofs}

    \subsection{Proof of \cref{lem:boundbregfromx1}} 

    \begin{proof}
        Note that $\breg{y}{x_1} \leq R(y) - R(x_1)$ since $\nabla R(x_1) \dotp (y - x_1) \ge 0$ by the first-order optimality criterion of $x_1$. Since $y = (1-\gamma)x + \gamma x_1$ for some $x \in \cS$,
        \[
            \pi_{x_1}(y)
            =
            \inf \brk[c]1{t > 0 : x_1 + t^{-1} (1-\gamma) (y - x_1) \in \cS}
            \le
            1-\gamma. 
        \]
        We now bound $R(y) - R(x_1)$ using \cref{eq:scupperbound}.
    \end{proof}

    \subsection{Proof of \cref{thm:mdbarrier}}
    
    For the proof we shall need the following fact about Bregman divergences.
    For any $x,y,z \in \text{int} \brk{\cS}$, it satisfies the following equation (easily shown):
    \begin{equation} \label{eq:bregmantriangle}
        \breg{y}{x} 
        = 
        \breg{y}{z} + \breg{z}{x} - \brk{\nabla R \brk{x} - \nabla R \brk{z}} \dotp \brk{y - z}
        .
    \end{equation}
    
    \begin{proof}
    First let us show that $\breg{u}{x_{t+1}'} \geq \breg{u}{x_{t+1}}$. 
    Note that $B_R$ is convex in its first argument, and $x_{t+1}$ minimizes $\breg{\,\cdot\,}{x_{t+1}'}$, entails $\brk{\nabla R(x_{t+1}) - \nabla R(x_{t+1}')} \dotp \brk{x_{t+1} - u} \le 0$, due to the first-order optimality of convex functions. Therefore, by \cref{eq:bregmantriangle},
    \begin{align}
        \breg{u}{x_{t+1}'}
        &=
        \breg{u}{x_{t+1}}
        + 
        \breg{x_{t+1}}{x_{t+1}'}
        -
        \brk{\nabla R(x_{t+1}) - \nabla R(x_{t+1}')} \dotp \brk{x_{t+1} - u} 
        \nonumber \\
        &\ge
        \breg{u}{x_{t+1}}
        , \label{eq:bregmanpythagorean}
    \end{align}
    by the first-order optimality criterion of the projection step and the non-negativity of the Bregman divergence.
    
    Next, we follow the standard mirror-descent analysis, reusing \cref{eq:bregmantriangle}, to obtain
    \begin{align*}
        \eta_t \ell_t \dotp (x_t - u) 
        &= 
        \brk{\nabla R(x_t) - \nabla R(x_{t+1}')} \dotp \brk{x_t - u}
        =
        \breg{u}{x_t} - \breg{u}{x_{t+1}'} + \breg{x_t}{x_{t+1}'}
        .
    \end{align*}
    Combining with \cref{eq:bregmanpythagorean} and summing over $t = 1,\ldots,T$:
    \begin{align*}
        \sum_{t=1}^T \ell_t \dotp (x_t - u) 
        &\le
        \sum_{t=1}^T \frac{1}{\eta_t} \brk{\breg{u}{x_t} - \breg{u}{x_{t+1}}}
        +
        \sum_{t=1}^T \frac{1}{\eta_t} \breg{x_t}{x_{t+1}'}
        ,
    \end{align*}
    where, using $\breg{u}{x_{T+1}} \geq 0$, 
    \[
        \sum_{t=1}^T \frac{1}{\eta_t} \brk{\breg{u}{x_t} - \breg{u}{x_{t+1}}}
        \le
        \frac{1}{\eta_1} \breg{u}{x_1}
        -
        \sum_{t=2}^T \brk3{\frac{1}{\eta_{t-1}} - \frac{1}{\eta_t}} \breg{u}{x_t}
        .
    \]
    
    Now denote $z = x_{t} - x_{t+1}'$. For the term $\breg{x_t}{x_{t+1}'}$, \cref{eq:scbergmanlb} entails that
    \begin{align*}
        \breg{x_t}{x_{t+1}'}
        &= 
        R(x_t) - R(x_{t+1}') - \nabla R(x_{t+1}') \dotp z \\
        &\le
        \brk1{\nabla R(x_{t})-\nabla R(x_{t+1}')} \dotp z - \rho\brk1{\norm{z}_{x_t}} \\
        &=
        \eta_t \ell_t \dotp z - \rho\brk1{\norm{z}_{x_t}} \\
        &\le
        \eta_t \norm{\ell_t}_{x_t}^\star \dotp \norm{z}_{x_t} - \rho \brk1{\norm{z}_{x_t}} 
        \tag{H\"{o}lder inequality} \\
        &\le
        \sup\nolimits_{\alpha \in \bbR} \{ \eta_t \norm{\ell_t}_{x_t}^\star \dotp \alpha - \rho(\alpha) \} \\
        &=
        \rho^\star \brk1{\eta_t \norm{\ell_t}_{x_t}^\star},
    \end{align*}
    where $\rho^\star$ is the Fenchel conjugate of $\rho$: $\rho^\star(x) = -x -\log(1-x)$ defined for any $x < 1$.
    The final statement is then given using $\rho^\star(x) \le x^2$ for any $x \in [0,1/2]$. 
    \end{proof}

    \subsection{Proof of \cref{thm:main}}
    
        \begin{restatement}[\cref{thm:main} (restated)]
        There exists an online algorithm for Finite-Horizon MDPs with Aggregated Feedback of expected regret,
            \[
                \bbE \brk[s]{\regret} 
                = \poly(|S|,|A|,H) \ \wt{O}(\sqrt{K}),
            \]
            in $K$ episodes.
        \end{restatement}

        In the remainder of this section we prove that the assumptions of the \blbacronym setting hold in each epoch with high probability, and bound the  constants $\beta, H, B$ (defined in \cref{sec:blb}).
        The following lemma quantifies how concentrated are our empirical estimates of the dynamics (\cref{eq:empirical-dynamics}) around the true values.
        \begin{lemma} \label{lem:concentration}
            With probability at least $1-\delta$, the following holds for all epochs $i=1,2,\ldots$ simultaneously:
            \begin{equation} \label{eq:hoff-conf-set}
                \norm{P \brk{\cdot \mid s,a,h} - \wh P_{i} \brk{\cdot \mid s,a,h}}_1
                \le
                5 \sqrt{ \frac{ |S| + \log \brk1{H |S| |A| K / \delta}}{\max \{N_i(s,a,h), 1\}}}
                ,
                \quad
                \forall (s,a,h) \in S \times A \times [H]
                .
            \end{equation}
        \end{lemma}
        To prove the lemma, we need the following simple technical result.
        \begin{lemma}[\citealp{weissman2003inequalities}] \label{thm:weissman}
            Let $p(\cdot)$ be a distribution over $m$ elements, and let $\bar{p}_t(\cdot)$ be the empirical distribution defined by $t$ i.i.d.~samples from $p(\cdot)$.
            Then, with probability at least $1 - \delta$,
            \[
                \| \bar{p}_t(\cdot) - p(\cdot) \|_1
                \le
                2 \sqrt{\frac{ m + \log \brk{ \delta^{-1}}}{t}}.
            \]
        \end{lemma}
        
        \begin{proof}[of \cref{lem:concentration}]
            Note that any state-action pair can be sampled at time $h$ during the episode at most $K$ times over the entire $K$ episodes. Then, the lemma from \cref{thm:weissman} and a union bound over all $(s,a,h) \in S \times A \times [H]$ and over all possible number of times in which $(s,a,h)$ can be sampled in total.
        \end{proof}
        
        Now, let $i$ be any epoch. 
        Before defining the set of feasible occupancy measures for epoch $i$, $\cS_i$, let us first simplify our notation.
        We write for any occupancy measure $x$,
        \[
            x(h,s,a) = \sum_{s' \in S} x(h,s,a,s'); \quad
            x(h,s) = \sum_{a \in A} x(h,s,a); \quad \text{and} \quad
            x(h) = \sum_{s \in S} x(h,s). 
        \]
        We define $\cS_i$ as follows:
        \begin{alignat}{2}
            \cS_i 
            = 
            \Bigl\{\, &x \in \bbR^{[H] \times S \times A \times S} ~:~ \nonumber \\
            & x(h,s,a,s') \ge 0, &&\forall (h,s,a,s') \in [H] \times S \times A \times S \label{eq:omnonnegative} \\
            & x(h) = 1, &&\forall h \in [H], \label{eq:omdistribution} \\
            & x(1,s) = \ind{s = s_1}, &&\forall s \in S. \label{eq:ominitial} \\
            & x(h+1,s) = \sum_{(s',a) \in S \times A} x(h,s',a,s), &&\forall (h,s) \in [H-1] \times S. \label{eq:omflow} \\
            & \norm{\wt P^{(x)}(\cdot \mid s,a,h) - \wh P_i(\cdot \mid s,a,h)}_1 
            \leq 
            \frac{\epsilon_i(s,a,h)}{H}, \quad &&\forall (h,s,a) \in [H] \times S \times A ~\Bigr\} \label{eq:omconcentration}.
        \end{alignat}
        \cref{eq:omnonnegative,eq:omdistribution,eq:ominitial,eq:omflow} simply define an occupancy measure, while \cref{eq:omconcentration} requires that the next-state distribution associated with the occupancy measure, $\wt P^{(x)}$ (\cref{eq:oc-pnd}), are close to the empirical next-state distribution (\cref{eq:empirical-dynamics}).
        The following lemma states that $\cS_i$ contains all occupancy measures associated with the true model dynamics~$P$.
        
        \begin{lemma} \label{lem:polycontainsopt}
            Suppose that \cref{eq:hoff-conf-set} holds, and let $x^\pi,P$ be an occupancy measure corresponding to some policy $\pi$ and the true model dynamics. Then $x \in \cS_i$.
        \end{lemma}
        
        \begin{proof}
            By definition of an occupancy measure, we have that \cref{eq:omnonnegative,eq:omdistribution,eq:ominitial,eq:omflow} hold, and that
            \[
                \wt P^{(x)}(s' \mid s,a,h) = P(s' \mid s,a,h), 
                \qquad \forall (h,s,a,s') \in [H] \times S \times A \times S,
            \]
            where $P$ is the true dynamics.
            Thus \cref{eq:omconcentration} holds by \cref{lem:concentration} and our claim follows.
        \end{proof}
        
        The next lemma bounds the difference in norm between any two occupancy measures in $\cS_i$ that correspond to the same policy
        (proof is lone and deferred to \cref{sec:proof-oc-diff-bound} below).
        
        \begin{lemma} \label{lem:boundocdiff}
           Suppose that \cref{eq:hoff-conf-set} holds, and let $x \in \cS_i$. 
            Let $x'$ be the occupancy measure of $\pi^{(x)}$ under the true model dynamics $P$. 
            Then, $\norm{x-x'}_1 \le \min\{\epsilon_i \dotp x, \epsilon_i \dotp x'\}$.
        \end{lemma}
        
        Lastly, note that according to the \blbacronym setting, one has to know an a-priori upper bound on $\sum_{t=1}^T (\hat z_t \dotp \epsilon_t)^2$. The bound is given by the following lemma. 
        
        \begin{lemma} \label{lem:sumofepsilons}
            Let $k_1, k_2,\ldots$ be such that $k_i$ is the initial episode for epoch $i$. Then, for every epoch $i$, 
            \[
                \sum_{k=k_i}^{k_{i+1}-1} \brk{\epsilon_i \dotp \hat z_k}^2 
                \le 
                25 H^4 |S| |A| \, \brk3{|S| + \log \frac{H |S| |A|K}{\delta}}.
            \]
        \end{lemma}
        
        \begin{proof}
            We have that $\hat{z}_k(s,a,h)$ is the empirical trajectory of episode $k\in [k_i,k_{i+1}-1]$. Therefore, $n_i(s,a,h)=\sum_{k=k_i}^{k_{i+1}-1} \hat{z}_k(s,a,h)$. Since during epoch $i$ we have $n_i(s,a,h)\leq \max\{N_i(s,a,h),1\}$, at the end of epoch $i$ we have $n_i(s,a,h)\leq \max\{N_i(s,a,h),1\}+1$, since the last trajectory might add $1$.
            %
             Also note that $\hat z_t$ is a vector whose elements are zero or one with exactly $H$ non-zeros. Therefore,
            \begin{align*}
                \sum_{k=k_i}^{k_{i+1}-1} \brk{\epsilon_i \dotp \hat z_k}^2
                &\le
                \sum_{k=k_i}^{k_{i+1}-1} H \sum_{\substack{(s,a,h) \\ \in S \times A \times [H]}} \hat z_k(s,a,h) \cdot \epsilon_i(s,a,h)^2 \\
                &=
                \sum_{\substack{(s,a,h) \\ \in S \times A \times [H]}} n_i(s,a,h) \cdot 25 H^3 \cdot \frac{ |S| + \log \brk{H |S| |A|K / \delta}}{\max \{N_i(s,a,h), 1\}} \\
                &\le
                \sum_{\substack{(s,a,h) \\ \in S \times A \times [H]}} \max\{N_i(s,a,h),1\} \cdot 25 H^3 \cdot \frac{ |S| + \log \brk{H |S| |A|K / \delta}}{\max \{N_i(s,a,h), 1\}} \\
                &\le
                \sum_{\substack{(s,a,h) \\ \in S \times A \times [H]}} 25 H^3 \, \brk3{|S| + \log \frac{H |S| |A|K}{\delta}} \\
                &\le
                25 H^4 |S| |A| \, \brk3{|S| + \log \frac{H |S| |A|K}{\delta}},
            \end{align*}
            where the first inequality is by Cauchy-Schwartz, the second is replacing the sum over  $ \hat z_k(s,a,h)$ by $ n_i(s,a,h)$, and the third uses the inequality  $n_i(s,a,h)\leq \max\{N_i(s,a,h),1\}$ from definition of \cref{alg:reduction}.
        \end{proof}

        We now prove the main theorem.
        
        \begin{proof}[of \cref{thm:main}]
        We run the algorithm of \cref{thm:blbregret} on $\cS_i$ in  epoch $i$, for every $i$, resetting the algorithm between epochs.
        \cref{thm:blbregret} bounds the expected regret in each epoch, whereas \cref{lem:concentration,lem:polycontainsopt,lem:boundocdiff} imply that the \blbacronym setting holds in each epoch with high probability. 
        
        To avoid having to deal with probabilistic dependencies, we only bound the expected regret.
        To do so, we can assume that there are exactly $2 H |S| |A| \log K$ epochs (by adding epochs with zero episodes), and that each epoch is run for exactly $K$ episodes (by padding with zero losses and the remaining episodes). 
          
        The analysis proceeds as follows. We set $\delta = 1/(H K)$, $\beta = 5 H \sqrt{|S| + \log(\ifrac{H|S||A| K}{\delta})}$, $B = \beta^2 |S| |A| H^2$, and $d = |S|^2 |A| H$.
        Recall that \cref{eq:hoff-conf-set} holds with probability at least $1-\delta$, and consider some epoch $i$. When \cref{eq:hoff-conf-set} holds, we have $x \in \cS_i$ by \cref{lem:polycontainsopt} as well as that $\norm{y_k-z_k}_1 \le \min\{y_k \dotp \epsilon_i, z_k \dotp \epsilon_i\}$ for all episodes $k$ during the epoch by \cref{lem:boundocdiff}.
        Moreover, we have that $\norm{\epsilon_i}_\infty \le \beta$ and that 
        $
            \sum_{k=k_i}^{k_{i+1}-1} \brk{\epsilon_i \dotp \hat z_k}^2 
            \le 
            B
        $ (\cref{lem:sumofepsilons}).
        Thus, conditioned on that \cref{eq:hoff-conf-set} holds up to epoch $i$ (which depends only on randomness prior to epoch $i$), the algorithm of \cref{thm:blbregret} obtains an expected regret bound in epoch $i$ of 
        \[
            \poly(d, \beta, H, B) \, O(\sqrt{K}) = \poly(H, |S|, |A|) \, \wt O(\sqrt{K}).
        \]
        If, on the other hand, \cref{eq:hoff-conf-set} does not hold, the regret in epoch $i$ is at most $H K$ which happens with probability at most $\delta$. Therefore, by the choice of $\delta$, we obtain that the expected regret in epoch $i$ is at most
        $
            \poly(|S|, |A|, H) \, \wt O \brk{\sqrt{K}}
            ,
        $
        where now the expectation is taken with respect to any randomness prior to the start of the epoch as well as during the epoch.
        
        We conclude the proof by summing over all epochs, which yields the final regret bound.
        \end{proof}

        \subsection{Proof of \cref{lem:boundocdiff}} \label{sec:proof-oc-diff-bound}
        
        \begin{proof}
        To simplify notation, we write
        \[
            x(h,s,a) = \sum_{s' \in S} x(h,s,a,s'), \quad \text{and} \quad 
            x(h,s) = \sum_{a \in A} x(h,s,a). 
        \]
        Define $\wt P(s'\mid s,a,h) = \frac{x(h,s,a,s')}{x(h,s,a)}$ and recall that $\pi_h(a \mid s) = \frac{x(h,s,a)}{x(h,s)}$.
        For $h=1$, we have
        \begin{align*}
            \sum_{\substack{(s,a,s') \\ \in S \times A \times S}} \abs1{x(1,s,a,s') - x'(1,s,a,s')} 
            &=
            \sum_{\substack{(s,a,s') \\ \in S \times A \times S}} \abs1{x(1,s) \wt P(s' \mid s, a, 1) - x'(1,s) P(s' \mid s, a, 1)} \pi_1(a \mid s) \\
            &=
            \sum_{\substack{(a,s') \\ \in A \times S}} \abs1{\wt P(s' \mid s_1, a,1) - P(s' \mid s_1, a, 1)} \pi_1(a \mid s_1) 
            \tag{\cref{eq:ominitial}} \\
            &\le
            \sum_{a \in A} \frac{\epsilon_i(s_1,a,1)}{H} \pi_1(a \mid s_1) 
            \tag{\cref{eq:omconcentration}} \\
            &\le
            \sum_{\substack{(s,a) \\ \in S \times A}} \frac{\epsilon_i(s,a,1)}{H} x'(1,s,a).
        \end{align*}

        Next, for $h > 1$,
        \begin{align*}
            &\sum_{\substack{(s,a,s') \\ \in S \times A \times S}} 
            \abs1{x(h,s,a,s') - x'(h,s,a,s')} \\
            &=
            \sum_{\substack{(s,a,s') \\ \in S \times A \times S}} 
            \abs1{x(h,s) \wt P(s' \mid s, a, h) - x'(h,s) P(s' \mid s, a, h)} 
            \cdot \pi_h(a \mid s) \\
            &=
            \sum_{\substack{(s,a,s') \\ \in S \times A \times S}} 
            \abs3{\sum_{\substack{(a'', s'') \\ \in A \times S}} 
            \brk2{x(h-1,s'',a'',s) \wt P(s' \mid s, a, h) - x'(h-1,s'',a'',s) P(s' \mid s, a, h)}} \; \pi_h(a \mid s) 
            \tag{\cref{eq:omflow}} \\
            &\leq
            \sum_{\substack{(s,a,s') \\ \in S \times A \times S}} 
            \abs3{\sum_{\substack{(a'', s'') \\ \in A \times S}} \brk1{x(h-1,s'',a'',s) - x'(h-1,s'',a'',s)}} \cdot \wt P(s' \mid s, a, h) \cdot \pi_h(a \mid s) \\
            &\qquad +
            \sum_{\substack{(s,a,s') \\ \in S \times A \times S}} 
            \abs3{\sum_{\substack{(a'', s'') \\ \in A \times S}} x'(h-1,s'',a'',s) 
            \brk2{\wt P(s' \mid s, a, h) - P(s'\mid s,a, h)}} \cdot \pi_h(a \mid s) \\
            &=
            \sum_{s \in S} \abs4{\sum_{\substack{(a'', s'') \\ \in A \times S}} \brk1{x(h-1,s'',a'',s) - x'(h-1,s'',a'',s)}} \\
            &\qquad +
            \sum_{\substack{(s,a,s') \\ \in S \times A \times S}} 
            \abs3{\sum_{\substack{(a'', s'') \\ \in A \times S}} x'(h-1,s'',a'',s) 
            \brk2{\wt P(s' \mid s, a, h) - P(s'\mid s,a, h)}} \cdot \pi_h(a \mid s) \\
            &\leq
            \sum_{\substack{(s,a'',s'') \\ \in S \times A \times S}} 
            \abs1{x(h-1,s'',a'',s) - x'(h-1,s'',a'',s)} \\
            &\qquad +
            \sum_{\substack{(s,a,s',a'',s'') \\ \in S \times A \times S\times A \times S}} x'(h-1,s'',a'',s)
            \abs2{\wt P(s' \mid s, a, h) - P(s'\mid s,a, h)} \cdot \pi_h(a \mid s) \\
            &\le
            \sum_{\substack{(s,a'',s'') \\ \in S \times A \times S}} 
            \abs1{x(h-1,s'',a'',s) - x'(h-1,s'',a'',s)} \\
            &\qquad +
            \sum_{\substack{(s,a,a'',s'') \\ \in S \times A \times A \times S}} x'(h-1,s'',a'',s) \cdot \frac{\epsilon_i(h,s,a)}{H} \cdot \pi_h(a \mid s) 
            \tag{\cref{eq:omconcentration}} \\
            &=
            \sum_{\substack{(s,a,s') \\ \in S \times A \times S}} 
            \abs1{x(h-1,s,a,s') - x'(h-1,s,a,s')}
            +
            \sum_{\substack{(s,a) \\ \in S \times A}} x'(h,s) \cdot \frac{\epsilon_i(h,s,a)}{H} \cdot \pi_h(a \mid s) 
            \tag{\cref{eq:omflow}} \\
            &=
            \sum_{\substack{(s,a,s') \\ \in S \times A \times S}} 
            \abs1{x(h-1,s,a,s') - x'(h-1,s,a,s')}
            +
            \sum_{\substack{(s,a) \\ \in S \times A}} x'(h,s,a) \cdot \frac{\epsilon_i(h,s,a)}{H}.
        \end{align*}
        Applying this argument recursively, we obtain
        \[
            \sum_{\substack{(s,a,s') \\ \in S \times A \times S}} 
            \abs1{x(h,s,a,s') - x'(h,s,a,s')}
            \le
            \frac{1}{H} \sum_{\substack{(h,s,a) \\ \in [H] \times S \times A}} x'(h,s,a) \cdot \epsilon_i(s,a,h)
            =
            \frac{x' \dotp \epsilon_i}{H},
        \]
        so that
        \[
            \norm{x-x'}_1 
            =
            \sum_{\substack{(h, s,a,s') \\ \in [H] \times S \times A \times S}} \abs1{x(h,s,a,s') - x'(h,s,a,s')}
            \leq
            x' \dotp \epsilon_i
            .
        \]
        A symmetric argument also provides $\norm{x-x'}_1 \le x \dotp \epsilon_i$. 
        \end{proof}

\subsection{Proof of \cref{thm:optimismmain}} \label{sec:opt-main-proof}

In this section we prove:

\begin{restatement}[\cref{thm:optimismmain} (restated)]
    Consider \cref{alg:inefficient} with $\eta = (2H\beta d)^{-1} \sqrt{\log |\cS| /T}$ and $\gamma = 2H^2 (H + \beta\sqrt{d}) \eta/\lambda$. 
    Then, given that $B \geq \sum_{t=1}^T (\hat z_t \dotp \epsilon_t)^2$ (almost surely), we have for any $y^\star \in \cS$:
    \begin{align*}
        \bbE\brk[s]4{ \sum_{t=1}^T \ell_t \dotp (\hat z_t - y^\star) }
        \leq 
        \brk*{ 4H\beta d + \frac{H^3}{\beta\lambda d} + \frac{H^2}{\lambda \sqrt{d}} } \sqrt{T \log \abs{\cS}}
        + 
        10\beta d \sqrt{B T}
        ,
    \end{align*}
    provided that $\beta \ge 1$ and $T \ge \ipfrac{4 H^2 (H + \beta \sqrt{d})^2 \log |\cS|}{\lambda^2 \beta^2 d^2}$.
\end{restatement}

The proof uses the following series of lemmas.
The following lemma argues that the regret of \cref{alg:inefficient} is bounded by the regret of the multiplicative weights updates, plus an additive error term that scales with the perturbations $\epsilon_t$.

\begin{lemma} \label{lem:to-mw-regret}
    Assume $\gamma \le \frac12$. For all $y^\star \in \cS$ it holds that 
    \begin{align*}
        \bbE\brk[s]4{ \sum_{t=1}^T \ell_t \dotp (\hat z_t - y^\star) }
        \leq
        \bbE\brk[s]4{ \sum_{t=1}^T \sum_{y \in \cS} p_t(y) \brk!{ \tilde\ell_t(y) - \tilde\ell_t(y^\star) } } 
        +
        \gamma H T
        +
        5d\,\bbE\brk[s]4{ \sum_{t=1}^T \norm{\epsilon_t}_{M_t} }
        .
    \end{align*}
\end{lemma}

\begin{proof}
    We prove that $\bbE_t \brk[s]{ \sum_{y \in \cS} p_t(y) \tilde\ell_t(y) } \geq \ell_t \dotp \hat z_t - 3d\,\norm{\epsilon_t}_{M_t}$ which, together with \cref{lem:loss-underestimate}, will imply the lemma by taking expectation and summing over $t=1,\ldots,T$.
    To see this, observe that by \cref{eq:tildell}, for all $y \in \cS$ one also has
    $
        \bbE_t\brk[s]{ \tilde\ell_t(y) }
        \geq
        \ell_t \dotp y - 2\sqrt{d} \norm{y}_{M_t^{-1}} \norm{\epsilon_t}_{M_t}
        ,
    $
    thus
    \[
        \bbE_t\brk[s]4{ \sum_{y \in \cS} p_t(y) \tilde\ell_t(y) }
        \geq
        \ell_t \dotp \sum_{y \in \cS} p_t(y) \, y - 2\sqrt{d} \, \bbE_t\brk[s]4{ \sum_{y \in \cS} p_t(y)
        \norm{y}_{M_t^{-1}} \norm{\epsilon_t}_{M_t} }
        .
    \]
    Now, $q_t = (1-\gamma)p_t + \gamma \mu$ together with $\gamma \le \frac12$ implies 
    $q_t - \gamma \mu \le p_t \le 2 q_t$. Therefore,
    (defining $x_t = \sum_{y \in \cS} q_t(y) \dotp y$)
    \begin{align*}
        \bbE_t\brk[s]4{ \sum_{y \in \cS} p_t(y) \tilde\ell_t(y) }
        &\geq
        \ell_t \dotp x_t 
        - 
        \gamma \sum_{y \in \cS} \mu(y) \, \ell_t \dotp y 
        - 
        4\sqrt{d} \, \bbE_t\brk[s]4{\sum_{y \in \cS} q_t(y) \norm{y}_{M_t^{-1}} \norm{\epsilon_t}_{M_t} }
        \\
        &\ge
        \ell_t \dotp x_t 
        - 
        \gamma H
        - 
        4\sqrt{d} \, \norm{\epsilon_t}_{M_t} \bbE_t\brk[s]1{ \norm{y_t}_{M_t^{-1}} }
        \\
        &\geq
        \ell_t \dotp x_t 
        - 
        \gamma H
        - 
        4d \, \norm{\epsilon_t}_{M_t}
        ,
    \end{align*}
    where the final inequality used $\bbE_t\brk[s]{ \norm{y_t}_{M_t^{-1}} } \leq \sqrt{ \bbE_t\brk[s]{ y_t\tr M_t^{-1} y_t } } = \sqrt{d}.$
    Finally, observe that 
    $
        (\bbE_t \abs{y_t \dotp \epsilon_t})^2
        \leq
        \bbE_t [(y_t \dotp \epsilon_t)^2]
        =
        \epsilon_t\tr \bbE_t[y_t y_t\tr] \epsilon_t
        =
        \norm{\epsilon_t}_{M_t}^2
        ,
    $
    so
    \begin{align*}
        \ell_t \dotp x_t
        &= 
        \bbE_t [\ell_t \dotp \hat z_t] + \bbE_t[\ell_t \dotp (y_t-z_t)]
        \\
        &\geq
        \bbE_t[\ell_t \dotp \hat z_t] - \bbE_t\abs{y_t \dotp \epsilon_t}
        \\
        &\geq
        \bbE_t[\ell_t \dotp \hat z_t] - \norm{\epsilon_t}_{M_t}
        .
    \end{align*}
    Thus we have
    \[
        \bbE_t\brk[s]4{ \sum_{y \in \cS} p_t(y) \tilde\ell_t(y) } 
        \ge 
        \bbE_t [\ell_t \dotp \hat z_t] - \gamma H - (4d + 1) \norm{\epsilon_t}_{M_t}
        \ge
        \bbE_t [\ell_t \dotp \hat z_t] - \gamma H - 5d \norm{\epsilon_t}_{M_t}
        .
    \]
    This concludes the proof.
\end{proof}

Next, we apply a standard second-order regret bound for the multiplicative weights method to obtain the following:

\begin{lemma} \label{lem:mw-sec-ord-bound}
Provided that 
$
    \gamma
    \geq 
    2H^2\max\brk[c]{H,\beta\sqrt{d}} \eta/\lambda,
$
the following regret bound holds for any $y^\star \in \cS$:
\begin{align*}
    \sum_{t=1}^T \sum_{y \in \cS} p_t(y) \brk!{ \tilde\ell_t(y) - \tilde\ell_t(y^\star) }
    \leq
    \frac{\log\abs{\cS}}{\eta} + \eta \sum_{y \in \cS} p_t(y) \brk{\tilde\ell_t(y)}^2
    .
\end{align*}
\end{lemma}

\begin{proof}
The claim would follow directly from the classical second-order bound for multiplicative weights (e.g., \citealp{cesa2007improved,dani2008price}) once we establish that $\abs{\tilde\ell_t(y)} \leq 1/\eta$ for all $t$ and $y \in \cS$.
Indeed, for all $t$ and $y$ we have
\begin{align*}
    \abs{ \tilde\ell_t(y) }
    &=
    \abs{(\ell_t \dotp \hat z_t) y\tr M_t^{-1} y_t - \sqrt{d} \, \norm{y}_{M_t^{-1}} \norm{\epsilon_t}_{M_t}}
    \\
    &\leq
    \abs{ \ell_t \dotp \hat z_t } \cdot \abs{ y\tr M_t^{-1} y_t } + \sqrt{d} \, \norm{y}_{M_t^{-1}} \norm{\epsilon_t}_{M_t}
    .
\end{align*}
Recall that $\abs{ \ell_t \dotp \hat z_t } \leq \norm{\ell_t}_\infty \norm{\hat z_t}_1 \le H$, $\norm{y} \leq H$ and $\norm{y_t} \leq H$ (see \cref{sec:blb}). Further,
$
    \norm{\epsilon_t}_{M_t}^2 
    = 
    \bbE_t[(y_t \dotp \epsilon_t)^2] \le (\beta H)^2
    .
$
Hence, we obtain that 
$
    \abs{ \tilde\ell_t(y) }
    \leq
    (H^3 + \beta H^2 \sqrt{d}) \norm{M_t^{-1}}
    .
$
To conclude, recall that $M_t \succeq \gamma \lambda I$ thanks to the added exploration,
so $\norm{M_t^{-1}} \le 1/(\lambda\gamma)$. Substituting this 
in
 the right-hand side and using the assumption that $
    \gamma
    \geq 
    2H^2\max\brk[c]{H,\beta\sqrt{d}} \eta/\lambda,
$ the desired bound on $|\tilde\ell_t(y)|$ follows.
\end{proof}

Finally, we establish a bound on the second-order variance term.

\begin{lemma} \label{lem:sec-mom-bound}
    Assume $\beta^2 d \ge 1$. 
    It holds that
    \begin{align*}
        \bbE_t\brk[s]4{ \sum_{y \in \cS} q_t(y) \tilde\ell_t(y)^2}
        \leq 
        (2H\beta d)^2
        . 
    \end{align*}
\end{lemma}

\begin{proof}
    Using the inequality $(a+b)^2\leq 2a^2 +2b^2$, we have
    \begin{align*}
        \bbE_t \brk[s]{ \tilde\ell_t(y)^2 } 
        =
        \bbE_t \brk[s]{ (\hat\ell_t \dotp y - \sqrt{d} \, \norm{y}_{M_t^{-1}} \norm{\epsilon_t}_{M_t})^2 }
        \leq
        2\bbE_t [( \hat\ell_t \dotp y)^2] + 2d \norm{y}_{M_t^{-1}}^2  \norm{\epsilon_t}_{M_t}^2
        .
    \end{align*}
    Now, for the first term we have
    \begin{align*}
        \bbE_t[( \hat\ell_t \dotp y)^2]
        =
        \bbE_t\brk[s]!{ (\ell_t \dotp \hat z_t)^2 \, y\tr M_t^{-1} y_t \, y_t\tr M_t^{-1} y }
        \leq
        H^2 \, y\tr M_t^{-1} \bbE_t\brk[s]{ y_t y_t\tr} M_t^{-1} y
        =
        H^2 y\tr M_t^{-1} y
        =
        H^2 \norm{y}_{M_t^{-1}}^2
        .
    \end{align*}
    For the second term, notice that
    \begin{align*}
        \norm{\epsilon_t}_{M_t}^2 
        = 
        \bbE_t[(y_t \dotp \epsilon_t)^2] 
        \leq (\beta H)^2
        . 
    \end{align*}
    Hence
    $
        \bbE_t \brk[s]{ \tilde\ell_t(y)^2 } 
        \leq
        2H^2\brk{ 1+d\beta^2 } \norm{y}_{M_t^{-1}}^2
        \leq
        4H^2\beta^2 d \norm{y}_{M_t^{-1}}^2
        ,
    $
    thus we can bound
    \begin{align*}
        \bbE_t\brk[s]4{ \sum_{y \in \cS} q_t(y) \brk{\tilde\ell_t(y)}^2 } 
        \leq
        4H^2\beta^2 d \sum_{y \in \cS} q_t(y) \norm{y}_{M_t^{-1}}^2 
        .
    \end{align*}
    To conclude, observe that
    \begin{align*}
        \sum_{y \in \cS} q_t(y) \norm{y}_{M_t^{-1}}^2 
        =
        \Tr\brk3{ M_t^{-1} \sum_{y \in \cS} q_t(y) yy\tr }
        =
        \Tr\brk{ M_t^{-1} M_t }
        =
        d
        .
        &\qedhere
    \end{align*}
\end{proof}

We can now prove \cref{thm:optimismmain}.

\begin{proof}
%
Combining \cref{lem:mw-sec-ord-bound,lem:to-mw-regret} and using \cref{lem:sec-mom-bound}, we have
\begin{equation} \label{eq:regret-with-params}
    \bbE\brk[s]4{ \sum_{t=1}^T \ell_t \dotp (\hat z_t - y^\star) }
    \leq
    \frac{\log \abs{\cS}}{\eta} 
    + 4(H\beta d)^2 \eta T
    + \gamma HT
    + 5d\,\bbE\brk[s]4{ \sum_{t=1}^T \norm{\epsilon_t}_{M_t} }
    .
\end{equation}

To bound the final term, we use two applications of Jensen's inequality,
\begin{align*}
    \bbE\brk[s]4{ \sum_{t=1}^T \norm{\epsilon_t}_{M_t} }
    \leq
    \sqrt{T \sum_{t=1}^T \bbE\norm{\epsilon_t}_{M_t}^2}
    =
    \sqrt{T \sum_{t=1}^T \bbE\brk[s]{(y_t \dotp \epsilon_t)^2}}
    .
\end{align*}
Further, observe that since $\norm{\epsilon_t}_\infty \leq \beta$ and $\norm{y_t-z_t}_1 \leq \abs{z_t \dotp \epsilon_t}$, we have
\begin{align*}
    \bbE\brk[s]{(y_t \dotp \epsilon_t)^2}
    &\leq
    2\bbE\brk[s]{(z_t \dotp \epsilon_t)^2} + 2\bbE\brk[s]{((y_t-z_t) \dotp \epsilon_t)^2}
    \\
    &\leq
    2\bbE\brk[s]{(z_t \dotp \epsilon_t)^2} + 2\bbE\brk[s]{\norm{y_t-z_t}_1^2 \norm{\epsilon_t}_\infty^2}
    \\
    &\leq
    2(1+\beta^2) \bbE\brk[s]{(z_t \dotp \epsilon_t)^2}
    ,
\end{align*}
and by Jensen's inequality we obtain
\[
    \bbE\brk[s]{(z_t \dotp \epsilon_t)^2}
    =
    \bbE\brk[s]{(\bbE_t[\hat z_t\mid z_t] \dotp \epsilon_t)^2}
    \le
    \bbE\brk[s]{(\hat z_t \dotp \epsilon_t)^2}.
\]
Thus,
\begin{align*}
    \bbE\brk[s]4{ \sum_{t=1}^T \norm{\epsilon_t}_{M_t} }
    \le
    \sqrt{T \cdot 4\beta^2 \sum_{t=1}^T \bbE\brk[s]{(\hat z_t \dotp \epsilon_t)^2}}
    \leq
    2\beta \sqrt{B T}
    .
\end{align*}
Plugging this into \cref{eq:regret-with-params}, and using the choices of $\eta$ and $\gamma$, the statement follows.
\end{proof}

\subsection{Proof of \cref{thm:inclrregret}}

Here we prove:

\begin{restatement}[\cref{thm:inclrregret} (restated)]
    Consider \cref{alg:increasinglearningrates} with
    \[
        \eta_0 
        = 
        \min \brk[c]3{\sqrt{\frac{\vartheta \log(H T)}{d^2 H^2 T}}, \, \frac{1}{4 d \sqrt{B T}}}
        .
    \] 
    Then, for any $y^\star \in \cS$ we have
    \[
        \bbE \brk[s]4{\sum_{t=1}^T (\hat z_t - y^\star) \dotp \ell_t}
        =
        O\brk2{
        d \beta H \vartheta \log (H T)
        +
        d \vartheta \sqrt{B T} \, \log (H T)
        + 
        d H \sqrt{\vartheta T \log(H T)}}
        ,
    \]
    provided that $B \geq \max\{\sum_{t=1}^T (\hat z_t \dotp \epsilon_t)^2, H\}$ (almost surely).
\end{restatement}

To prove the theorem, we first prove a few lemmas that will aid in the main proof. Our first lemma shows some necessary technical results, the first of which is that indeed $y_t \in \cS$ for all $t=1,\ldots,T$.

\begin{lemma} \label{lem:technicalstuff}
    For all $t=1,\ldots,T$:
    \, $y_t \in \cS$; \,
    \, $\norm{\nabla^2 R(x_t)^{1/2} u_t}_{x_t}^\star=1$; \, and \,
    \, $\norm{\tilde \ell_t}_{x_t}^\star \le d H$.
\end{lemma}

\begin{proof}
    Since $R$ is a self-concordant barrier function over 
    a compact set $\cS$, following \cref{eq:dikininbody}, it suffices to show that for all $t$, $\norm{y_t-x_t}_{x_t} \le 1$, and indeed
    \[
        \norm{y_t-x_t}_{x_t}^2
        =
        \norm{\nabla^2 R(x_t)^{-1/2} u_t}_{x_t}^2
        =
        u_t\tr \nabla^2 R(x_t)^{-1/2} \nabla^2 R(x_t) \nabla^2 R(x_t)^{-1/2} u_t
        = 
        1. 
    \]
    Similarly,
    \[
        \brk1{\norm1{\nabla^2 R(x_t)^{1/2} u_t}_{x_t}^\star}^2 =  u_t\tr \nabla^2 R(x_t)^{1/2} \nabla^2 R(x_t)^{-1} \nabla^2 R(x_t)^{1/2} u_t =1,
    \]
    and
    \[
        \norm{\tilde \ell_t}_{x_t}^\star
        = 
        d \, \abs{\ell_t \dotp \hat z_t} \, \norm{\nabla^2 R(x_t)^{1/2} u_t}_{x_t}^\star
        \le
        d H. \qedhere
    \]
\end{proof}

\begin{lemma} \label{lem:lr-relation}
    Suppose $\eta_0 \le 1 / 4 d \sqrt{B T}$, then
    $\eta_0 \le \eta_t \le 2 \eta_0, \quad \forall t=1,\ldots,T$.
\end{lemma}

\begin{proof}
    $\eta_0 \le \eta_t$ holds by definition. The other direction is because
    \[
        \eta_t^{-1} 
        = 
        \eta_0^{-1} - 2 d \sum_{s=1}^t |\hat z_t \dotp \epsilon_t|
        \ge
        \eta_0^{-1} - 2 d \sqrt {T \dotp \sum_{s=1}^t \brk{\hat z_t \dotp \epsilon_t}^2}
        \ge
        \eta_0^{-1} - 2 d \sqrt{B T}
        \ge
        \half \eta_0^{-1}. \qedhere
    \]
\end{proof}

Finally, we combine the lemmas above with the guarantee of OMD to yield the main theorem.

\begin{proof}[of \cref{thm:inclrregret}]
    Observe the three summands of \cref{eq:regretdecomp}.
    For the first summand, we have
    \begin{align*}
        \bbE \brk[s]4{\sum_{t=1}^T (z_t-x_t) \dotp \ell_t}
        &=
        \bbE \brk[s]4{\sum_{t=1}^T (z_t-y_t) \dotp \ell_t}
        \le
        \bbE \brk[s]4{\sum_{t=1}^T |z_t \dotp \epsilon_t|} \\
        &\le
        \bbE \brk[s]4{\sum_{t=1}^T |\hat z_t \dotp \epsilon_t|} 
        \le
        \sqrt{T \, \bbE \brk[s]4{\sum_{t=1}^T \brk{\hat z_t \dotp \epsilon_t}^2}} 
        \le
        \sqrt{B T}
        ,
    \end{align*}
    where the first inequality uses that $\|\ell_t\|_\infty \leq 1$ and the assumption that $\norm{z_t-y_t}_1\leq |z_t \dotp \epsilon_t|$, the second inequality is by Jensen's inequality, and the third inequality is due to Cauchy-Schwartz. For the last inequality we recall that $\sum_{t=1}^T \brk{\hat z_t \dotp \epsilon_t}^2 \le B$ by the assumptions of the \blbacronym setting (see \cref{sec:blb}).

    %
    For the second summand in \cref{eq:regretdecomp}, since $B \ge H$, by our choice of $\eta_0$, and by \cref{lem:lr-relation} we have $\eta_t \le 2 \eta_0 \le 1 / 2 d H$, so $\eta_t \|\tilde \ell_t\|_{x_t}^\star \le \frac12$ by \cref{lem:technicalstuff}.
    We can therefore apply \cref{thm:mdbarrier} to get
    \[
        \bbE \brk[s]4{\sum_{t=1}^T (x_t - y^\star_\gamma) \dotp \tilde \ell_t} 
        \leq
        \bbE \brk[s]4{
        \underbrace{\vphantom{\sum_{t=2}^T}\frac{1}{\eta_1} \breg{y^\star_\gamma}{x_1}}_{(1)}
        -
        \underbrace{\sum_{t=2}^T \brk3{\frac{1}{\eta_{t-1}} - \frac{1}{\eta_t}} \breg{y^\star_\gamma}{x_t}}_{(2)}
        + 
        \underbrace{\sum_{t=1}^T \eta_t (\|\tilde \ell_t\|_{x_t}^\star)^2}_{(3)}
        }.
    \]
    We now bound each of the three terms $(1),(2)$, and $(3)$.
    We have
    $
        (1) 
        \le 
        \eta_0^{-1} \vartheta \log(\gamma^{-1})
    $
    by \cref{lem:boundbregfromx1} and as $\eta_1 \ge \eta_0$ (\cref{lem:lr-relation}).
    For term $(2)$, we have
    \begin{align*}
        (2) 
        &=
        2 d \sum_{t=2}^T |\hat z_t \dotp \epsilon_t| \breg{y^\star_\gamma}{x_t} \\
        &=
        2 d \sum_{t=1}^T |\hat z_t \dotp \epsilon_t| \breg{y^\star_\gamma}{x_t}
        -
        2 d \underbrace{|\hat z_1 \dotp \epsilon_1|}_{\mathclap{\le \|\hat z_1\|_1 \|\epsilon_1\|_\infty \le H \beta}} \breg{y^\star_\gamma}{x_1} \\
        &\ge
        2 d \sum_{t=1}^T |\hat z_t \dotp \epsilon_t| \brk3{\half \norm{x_t - y^\star_\gamma}_{x_t} - 1}
        -
        2d \beta H \cdot \vartheta \log \frac{1}{\gamma} 
        \tag{\cref{lem:boundbregfromx1,lem:boundbregfromxt}}
        \\
        &=
        d \sum_{t=1}^T |\hat z_t \dotp \epsilon_t| \norm{x_t - y^\star_\gamma}_{x_t}
        -
        2 d \sum_{t=1}^T |\hat z_t \dotp \epsilon_t|
        -
        2d \beta H \cdot \vartheta \log \frac{1}{\gamma} \\
        &\ge
        d \sum_{t=1}^T |\hat z_t \dotp \epsilon_t| \norm{x_t - y^\star_\gamma}_{x_t}
        -
        2 d \sqrt{B T}
        -
        2d \beta H \cdot \vartheta \log \frac{1}{\gamma}
        ,
    \end{align*}
    where the last inequality is since $\sum_{t=1}^T |\hat z_t \dotp \epsilon_t| \le \sqrt{T \sum_{t=1}^T (\hat z_t \dotp \epsilon_t)^2}$ by Cauchy-Schwartz and as $\sum_{t=1}^T (\hat z_t \dotp \epsilon_t)^2 \le B$ by assumption.
    We lastly employ \cref{lem:technicalstuff} and that $\eta_t \le 2 \eta_0$ by \cref{lem:lr-relation} to bound
    $
        (3) 
        \le
        2 \eta_0 d^2 H^2 T
        .
    $
    All in all, this obtains us \cref{eq:xxxxx}.
    
    We sum \cref{eq:specificlossbias} over all $t$ and take expectation. Together with \cref{eq:xxxxx} this replaces the perceived losses, $\tilde{\ell}_t$, by the real losses, $\ell_t$.
    The terms
    $
        d \, \bbE \brk[s]{\sum_{t=1}^T |\hat z_t \dotp \epsilon_t| \, \|x_t - y^\star\|_{x_t}}
    $
    in \cref{eq:specificlossbias} and in \cref{eq:xxxxx} cancel out, and we get
    \[
        \bbE \brk[s]4{\sum_{t=1}^T (x_t - y^\star_\gamma) \dotp \ell_t }
        \le 
        \brk3{\frac{1}{\eta_0} + 2 d \beta H} \vartheta \log \frac{1}{\gamma}
        +
        2 d \sqrt{B T}
        + 
        2 \eta_0 d^2 H^2 T.
    \]
    
    Finally, for the third summand in \cref{eq:regretdecomp}, we have
    \[
        \sum_{t=1}^T (y^\star_\gamma-y^\star) \dotp \ell_t = \gamma \sum_{t=1}^T (x_1 - y^\star) \dotp \ell_t \le 2 \gamma H T.
    \]
    Combining the bounds on all three summands and setting $\gamma, \eta_0$ as in the theorem's statement yields the final regret bound. 
\end{proof}

\section{Self-concordant Barriers: definitions and basic properties}
\label{app:self}

For a $k$-array tensor $U\in\mathbb{R}^{d\times k}$, we define 
$$
    U[h_1, \ldots , h_k]=\sum_{i_1, \ldots,i_k\in [d]} U(i_1, \ldots ,i_k) \prod_{j=1}^k h_j(i_j)
    .
$$
For $k=2$ we have that $U$ is a matrix, $h_1$ and $h_2$ are vectors, and $U[h_1,h_2]=h_1\tr U h_2$.

\begin{definition}
    For a convex set $\cS \subset \bbR^n$, a self-concordant function $R : \text{int}(\cS) \mapsto \bbR$ is a $C^3$-convex function such that
    \[
        \abs1{D^3 R(x)[h,h,h]} \le 2 \brk1{D^2 R(x)[h,h]}^{3/2}.
    \]
    In words: the third derivative of $R$ at $x$ in direction $h$ is upper bounded by a constant times the second derivative of $R$ at $x$ in direction $h$,
    raised to the $3/2$ power.
\end{definition}

\begin{definition}
    A self-concordant function $R$ is a $\vartheta$-self-concordant barrier if
    \[
        \abs1{D R(x)[h]} \le \vartheta^{1/2} \brk1{D^2 R(x)[h,h]}^{1/2}.
    \]
\end{definition}

We have the following upper bound on the difference a $\vartheta$-self-concordant barrier $R$ at two points $x,y \in \cK$:
\begin{equation} \label{eq:scupperbound}        
    R(y) - R(x) \le \vartheta \log \frac{1}{1- \pi_{x}(y)}
    ,
\end{equation}
where $\pi_x(y)$ is the Minkowski function of $\cS$ w.r.t.~$x$: $\pi_{x}(y) = \inf \brk[c]{t > 0 : x + t^{-1}(y - x) \in \cS}$. 

\section{Efficient Implementation of the Reduction} \label{app:efficientreduction}
    
    In this section we complete the proof of \cref{thm:main} by showing a computationally-efficient reduction between Finite-Horizon MDPs with Aggregate Feedback and that of \blbname.
    
    Recall the reduction in \cref{sec:reduction} in which we showed how to solve a Finite-Horizon MDPs with Aggregated Feedback by constructing a sequence of $O(\log K)$ instances (epochs) of the \blbname\ problem and running a no-regret algorithm in each such instance (which exists due to \cref{thm:blbregret}). In subsequent sections we reviewed \cref{alg:inefficient,alg:increasinglearningrates}, both of which guarantee no-regret for \blbacronym. 
    In this section we make the choice of the algorithm for the reduction explicit by fixing it to be \cref{alg:increasinglearningrates}.
    Note that the reduction itself, as well as \cref{alg:increasinglearningrates}, can be implemented in polynomial-time  as long as in each epoch $i$, the barrier $R$ chosen for $\cS_i$ can be computed efficiently.
    However, \cref{alg:increasinglearningrates} is made for the case in which $\cS$ has volume in $\bbR^d$ which is not the case of our body $\cS_i$. Thus, in what follows we give two options on how to alleviate this problem and build an efficiently-computable barrier function for each option. In option 1, we show how to alter \cref{alg:increasinglearningrates} to accommodate the case for $\cS_i$ not being fully-dimensional. In option 2, we keep \cref{alg:increasinglearningrates} as it is, but change the reduction so that $\cS_i$ has a small volume in $\bbR^d$
    
    \paragraph{Option 1.}

    We follow a technique used in \cite{lee2020bias}.
    The set $\cS_i$ consists of an intersection between linear equations (\cref{eq:omdistribution,eq:omflow,eq:ominitial}) of the form $c_i \dotp x = d_i$ for $i=1,\ldots,p$ and linear inequalities (\cref{eq:omnonnegative,eq:omconcentration}) of the form $a_i \dotp x \le b_i$ for $i=1,\ldots,m$ where $m = O(|S|^2 H |A|)$.
    Our approach is to set the log barrier $R(x) = -\sum_{i=1}^m \log(b_i - a_i \dotp x)$ over the inequalities (note that its barrier parameter $\vartheta$ is $m$; see \citealp{nemirovski2004interior}).
    However, we still have to handle the linear equations in order to make sure that \cref{alg:increasinglearningrates} will not generate predictions that are not in $\cS_i$.
    
    Recall that \cref{alg:increasinglearningrates} is essentially a variant of OMD, which commonly has a projection step that does not appear in \cref{alg:increasinglearningrates}. 
    First, we add a projection step in \cref{alg:increasinglearningrates} after \cref{ln:update} onto the affine subspace defined by the linear equations of $\cS_i$: we replace \cref{ln:update} with $x_{t+1}' = \nabla R^{-1}\brk{\nabla R(x_t) - \eta_t \tilde \ell_t}$, and then add after \cref{ln:update}: $x_{t+1} = \argmin_{x : C x = d} \breg{x}{x_{t+1}'}$, where $C$ is a matrix whose columns are $c_1,\ldots,c_p$.
    This, in turn, validates that the iterates $x_1,x_2,\ldots$ are in $\cS_i$. 

    Second, recall that originally $y_k$, is sampled uniformly at random from the Dikin ellipsoid centered at $x_t$: $\{y : \|y - x_t\|_{x_t} \le 1\}$. 
    Concretely, $y_k = x_k + \nabla^2 R(x_k)^{-1/2} u_k$ (\cref{ln:yt}) for $u_k$ sampled uniformly at random from the unit sphere of $\bbR^d$. We also like to make sure that $y_k$ is in the aforementioned affine subspace, by instead sampling $y_k$ uniformly at random from the intersection of the Dikin ellipsoid with the affine subspace.
    To achieve this, we let $W$ be an orthogonal matrix whose range spans the null space of $C$. We now sample $u_k$ uniformly from the unit sphere in $\bbR^p$. 
    We replace \cref{ln:yt} in \cref{alg:increasinglearningrates} by choosing $y_k = x_k + W W\tr \nabla^2 R(x_k)^{-1/2} W u_k$, so now $y_k - x_k$ is in the null space spanned by $c_1,\ldots,c_p$. Moreover, we have $y_k \in \cS_i$ due to (see \cref{lem:technicalstuff}):
    \begin{align*}
        \norm{y_k - x_k}_{x_k}^2 
        &=
        (y_k - x_k)\tr \nabla^2 R(x_k) (y_k - x_k) \\
        &= 
        u_k\tr W\tr \nabla^2 R(x_k)^{-1/2} W W\tr \nabla^2 R(x_k) W W\tr \nabla^2 R(x_k)^{-1/2} W u_k \\
        &=
        u_k\tr \brk{W\tr \nabla^2 R(x_k) W}^{-1/2} \brk{W\tr \nabla^2 R(x_k) W} \brk{W\tr \nabla^2 R(x_k) W}^{-1/2} u_k 
        \tag{$W$ is orthogonal}
        \\
        &=
        1
        .
    \end{align*}
    The estimators $\brk[c]{\tilde \ell_t}_{t=1}^T$ have to be changed accordingly. We change \cref{ln:estimator} by redefining
    $
        \tilde \ell_t
        = 
        p (\ell_t \dotp \hat z_t) W W\tr \nabla^2 R(x_t)^{1/2} W u_t
        ,
    $
    Following which, we alter the rest \cref{lem:technicalstuff} as follows:
    \begin{align*}
        \brk1{\norm{W W\tr \nabla^2 R(x_t)^{1/2} W u_t}_{x_t}^\star}^2 
        &=  
        u_t\tr W\tr \nabla^2 R(x_t)^{1/2} W W\tr \nabla^2 R(x_t)^{-1} W W\tr \nabla^2 R(x_t)^{1/2} W u_t \\
        &=
        u_t\tr \brk{W\tr \nabla^2 R(x_t) W}^{1/2} \brk{W\tr \nabla^2 R(x_t) W}^{-1} \brk{W\tr \nabla^2 R(x_t) W}^{1/2} u_t 
        \\
        &=
        u_t\tr u_t
        =
        1
        , 
    \end{align*}
    where the second equality is as $W$ is orthogonal.
    Moreover,
    \[
        \norm{\tilde \ell_t}_{x_t}^\star
        = 
        p \, \abs{\ell_t \dotp \hat z_t} \, \norm{W W\tr \nabla^2 R(x_t)^{1/2} W u_t}_{x_t}^\star 
        \le
        p H.
    \]
    The proof of \cref{lem:boundlossbias} is changed in the same manner, using the fact that $x_k - x$ is in the span of $W$. 
    The rest of the proof of the analysis of \cref{alg:increasinglearningrates} remains without any further changes.
    
    \paragraph{Option 2.}
    
    In this option, instead of altering \cref{alg:increasinglearningrates}, we alter $\cS_i$ to give it a small volume in $\bbR^d$.
    We replace the $p$ linear equations of the form $c_i \dotp x = d_i$ with linear inequalities of the form $\abs{c_i \dotp x - d_i} \le 1/\poly(K)$. We then set the barrier on the new body to be the log barrier of the new set of linear inequalities:
    \[
        R(x) 
        = 
        -\sum_{i=1}^m \log(b_i - a_i \dotp x) 
        -\sum_{i=1}^p \log (\poly(K)^{-1} - \abs{d_i - c_i \dotp x})
        ,
    \]
    which also has a barrier parameter of $\vartheta = O(|S|^2 |A| H)$ (number of linear inequalities defining the new body; see \citealp{nemirovski2004interior}). 
    
    The issue here is that, when running \cref{alg:increasinglearningrates} on the new body, we might choose $y_k$ that is on the exterior of $\cS_i$. 
    However, we could then replace $y_k$ by its projection onto $\cS_i$ and play that projection instead which we denote by $y'_k$. Note that $\norm{y_k - y'_k}_1 \le O(1/\poly(K))$.
    This ensures that $\norm{y_k - z_k}_1 \le \abs{\epsilon_i \dotp z_k} + O(1/\poly(K))$ which suffices to fulfill the assumptions of the \blbname\ setting (\cref{sec:blb}) thus ensuring that \cref{alg:increasinglearningrates} will maintain its $\wt O(\sqrt{K})$ regret bound.

\end{document}